\documentclass{article}

\usepackage{PRIMEarxiv}

\usepackage[utf8]{inputenc} % allow utf-8 input
\usepackage[T1]{fontenc}    % use 8-bit T1 fonts
\usepackage{hyperref}       % hyperlinks
\usepackage{url}            % simple URL typesetting
\usepackage{booktabs}       % professional-quality tables
\usepackage{amsfonts}       % blackboard math symbols
\usepackage{nicefrac}       % compact symbols for 1/2, etc.
\usepackage{microtype}      % microtypography
\usepackage{lipsum}
\usepackage{fancyhdr}       % header
\usepackage{graphicx}       % graphics
\graphicspath{{media/}}     % organize your images and other figures under media/ folder

\usepackage{colortbl} % 需要在导言区加
\usepackage{venndiagram}
\usepackage{algorithm}
\usepackage{algorithmic}
\usepackage{amsmath}
\usepackage{bm}
\usepackage{booktabs}
\usepackage{multirow}
\usepackage{makecell}
\usepackage{braket}
\usepackage{amsthm}
\usepackage{xcolor}

\newtheorem{claim}{Claim}
\title{Q-RUN: Quantum-Inspired Data Re-uploading Networks}

\title{Q-RUN: Quantum-Inspired Data Re-Uploading Networks}

\author{
Wenbo Qiao\textsuperscript{1}, 
Shuaixian Wang\textsuperscript{1}, 
Peng Zhang\textsuperscript{2}\thanks{Corresponding author} , 
Yan Ming\textsuperscript{1}, 
Jiaming Zhao\textsuperscript{1} \\
\textsuperscript{1}School of New Media and Communication, Tianjin University, Tianjin, China \\
\textsuperscript{2}College of Intelligence and Computing, Tianjin University, Tianjin, China
}

\begin{document}
\maketitle

\begin{abstract}
Data re-uploading quantum circuits (DRQC) are a key approach to implementing quantum neural networks and have been shown to outperform classical neural networks in fitting high-frequency functions. However, their practical application is limited by the scalability of current quantum hardware. In this paper, we introduce the mathematical paradigm of DRQC into classical models by proposing a quantum-inspired data re-uploading network (Q-RUN), which retains the Fourier-expressive advantages of quantum models without any quantum hardware.
Experimental results demonstrate that Q-RUN delivers superior performance across both data modeling and predictive modeling tasks. Compared to the fully connected layers and the state-of-the-art neural network layers, Q-RUN reduces model parameters while decreasing error by approximately one to three orders of magnitude on certain tasks. Notably, Q-RUN can serve as a drop-in replacement for standard fully connected layers, improving the performance of a wide range of neural architectures. This work illustrates how principles from quantum machine learning can guide the design of more expressive artificial intelligence.
\end{abstract}

\section{Introduction}
The multilayer perceptron (MLP) remains one of the most fundamental and essential components in artificial intelligence (AI) \cite{lecun2015deep}. As a fundamental building block, MLPs are widely used in constructing a variety of neural models, including convolutional neural networks (CNNs), long short-term memory networks (LSTMs), and large language models, among others. 
However, traditional MLPs face two major limitations, as illustrated in Figure~\ref{fig1}(a). First, the frequency principle indicates that MLPs, especially those using ReLU activation functions, naturally tend to prioritize fitting low-frequency features, which restricts their ability to capture high-frequency components \cite{xu2019frequency,rahaman2019spectral}. Second, the fully connected structure of MLPs causes a parameter explosion and high computational costs, with diminishing returns on scaling as Moore’s Law slows down \cite{thompson2020computational}. These challenges prompt us to consider whether a network structure exists that can effectively model high-frequency information with lower computational cost.

% This insight has spurred a growing body of work on MLP variants that incorporate novel activation functions or network architectures. One promising approach introduces inductive biases related to Fourier series fitting, which enhance the representation of high-frequency signals and have demonstrated superior performance over classical MLPs in specific domains. However, their capacity to express Fourier components typically scales only linearly with model size, meaning that modeling complex tasks still requires a large number of parameters.

\begin{figure}[t]
    \scriptsize
    \centering
    \includegraphics[width=0.6\linewidth]{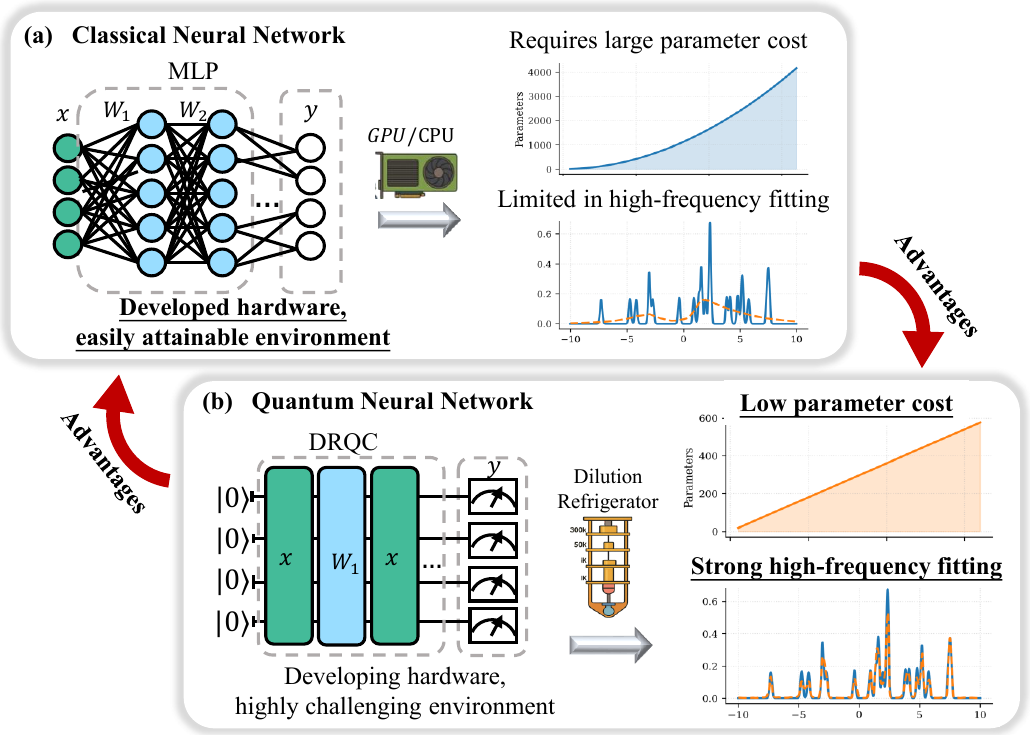}
    \caption{Comparison of classical and quantum neural networks. The underlines highlight the respective strengths of the two computational paradigms.}
        \label{fig1}
\end{figure}

Recently, quantum machine learning (QML) has gained attention for leveraging quantum parallelism and high-dimensional Hilbert spaces to potentially boost the speed or accuracy of classical models \cite{biamonte2017quantum,chen2024hands}. 
Among them, quantum neural networks (QNN) based on the data re-uploading quantum circuit (DRQC) \cite{perez2020data} have demonstrated exponential expressive power in fitting truncated Fourier series \cite{schuld2021effect, yu2022power}. This enables QNNs to efficiently capture high-frequency components of natural continuous signals, such as images and sound waves \cite{zhao2024quantum}, and to perform time series prediction \cite{qiao2025quantum}, with fewer parameters. This suggests that QNN possesses a more lightweight and efficient inductive bias for representing Fourier series.
However, due to the current limitations of noisy intermediate-scale quantum (NISQ) devices, such DRQC remain difficult to scale up on real quantum hardware, as illustrated in Figure~\ref{fig1}(b). Given these observations, a natural question arises:

\begin{center}
\emph{``Can we leverage QML principles to design better AI models within today’s classical frameworks?''}
\end{center}

Fortunately, we observe that the advantage of DRQC lies primarily in modeling capacity rather than computational acceleration. This insight implies that we may leverage the core inductive biases of such quantum models without relying on quantum hardware. In this paper, we propose a quantum-inspired data re-uploading network (Q-RUN) that leverages the Fourier series representation capabilities of DRQC. This approach not only tackles high-frequency fitting but also applies to a wide range of tasks with fewer parameters than traditional fully connected networks, without relying on quantum hardware.
Specifically, we first introduce a version of Q-RUN that strictly follows the principles of quantum computation. It leverages data re-uploading to encode the input features and employs quantum measurement to extract relevant information, thereby incorporating the inductive bias of quantum models that fit Fourier series into the network. Then, to enable more efficient execution of Q-RUN on classical hardware, we relax the strict rules of quantum computation by implementing data re-uploading through feature concatenation and introducing a small element-wise shared MLP to approximate the quantum measurement process.

We demonstrate that Q-RUN outperforms fully connected layers (FCs) and other state-of-the-art network layers in various data modeling tasks (e.g., implicit neural representations, probabilistic density estimation, and molecular energy modeling), achieving improvements of one to three orders of magnitude on certain tasks while using fewer parameters. Moreover, in predictive modeling tasks (e.g., time-series forecasting, language modeling, image recognition, and parameter-efficient fine-tuning of pretrained language models), Q-RUN exhibits plug-and-play compatibility and improves performance across diverse architectures, illustrating the broad applicability of the inductive bias introduced by quantum models.
\textbf{Our contributions are as follows:}
\begin{itemize}
    \item We propose Q-RUN, which captures the unique inductive bias of quantum models, achieving strong capability in modeling Fourier series and effectively fitting high-frequency functions with fewer parameters.
    \item We validate Q-RUN across a wide range of data modeling and predictive modeling tasks, demonstrating strong competitiveness and plug-and-play applicability as a replacement for standard FCs in network architectures.
    \item Q-RUN illustrates a practical approach to harnessing the advantages of QML for classical AI, offering a promising path to exploiting QML principles in the near term.
\end{itemize}

The remainder of this paper is organized as follows. 
In Section~2, we review related work and position our contributions within the existing literature. 
Section~3 provides the necessary preliminaries on quantum machine learning. 
In Section~4, we present the theoretical foundation of Q-RUN together with its practical model architecture. 
Section~5 reports empirical evaluations of Q-RUN across a variety of tasks. 
Finally, Section~6 concludes the paper and outlines future research directions.

\section{Related Work}

Q-RUN is a classical neural network inspired by QNN, situating it at the intersection of classical and quantum models, as illustrated in Figure \ref{frw}. This section reviews related work from both areas.
% Q-RUN 是一个由量子神经网络启发的经典神经层，因此它的相关工作落在经典神经网络和量子神经网络的交叉区间之上。在本节我们将分别介绍这两个方向。

\subsection{Classical Neural Network}

MLPs serve as the backbone of many complex modern architectures. However, the frequency principle in MLPs \cite{xu2019frequency} and the slowdown of Moore’s Law \cite{thompson2020computational} have motivated the search for new foundational architectures.
One promising direction introduces inductive biases related to Fourier series fitting, which improve the model's ability to represent high-frequency signals and have shown superior performance over vanilla MLPs in certain domains \cite{rahaman2019spectral,dong2024fan}. For example, MLPs with sinusoidal activations (SIREN) have been widely used in implicit representations, enabling efficient storage of images and audio, as well as modeling of 3D scenes \cite{sitzmann2020implicit}. 
Another related approach, Random Fourier Features (RFF), first projects input data onto a random Fourier basis before feeding it into the network, thereby enhancing its ability to capture high-frequency patterns \cite{tancik2020fourier}. This approach has shown strong promise in time series forecasting \cite{woo2023learning} and positional encoding for large language models \cite{hua2024fourier}.
However, because their Fourier representation capacity grows only linearly with model size, representing high-dimensional Fourier series demands exponentially more parameters \cite{zhao2024quantum}.

Another MLP variant class uses activation functions suited to specific tasks instead of Fourier bases. For example, piecewise linear neural networks (PWLNNs) use piecewise activation functions to approximate target functions and naturally extend ReLU-based networks, effectively modeling complex nonlinear systems \cite{tao2022piecewise}. Similarly, Kolmogorov–Arnold Networks (KANs) utilize B-spline basis functions combined with extensive tuning strategies to approximate target \cite{liu2024kan}. These methods often use fewer parameters, showing promise in symbolic regression and solving partial differential equations. However, their success is mostly domain-specific, and their broader applicability to AI remains under study.

% 基于各种非线性激活函数的MLP是经典神经网络的最初形态，也是目前各种复杂网络结构的基础。但是频率原理和摩尔定理的放缓都在要求我们寻找新的基础网络层。
%  This insight has spurred a growing body of work on MLP variants that incorporate novel activation functions or network architectures. One promising approach introduces inductive biases related to Fourier series fitting, which enhance the representation of high-frequency signals and have demonstrated superior performance over 朴素 MLPs in specific domains. 利如使用sin激活函数的MLP，该方法被广泛应用于隐式神经表征，能够高效的存储图片、声波，建模3D信号。甚至在生物科学、材料化学上也展现出应用前景。
%  类似的思想还有在数据输入网络之前先将数据映射到一个随机傅里叶序列，然后提高对高频信号的预测。这种方法在时间序列预测、大模型的位置编码上都展现出了应用前景。还有同时sine cos 以及gelu激活函数的方法，该方法可以广泛的提升各类人工智能任务。

%  However, their capacity to express Fourier components typically scales only linearly with model size. 反过来说这 当面对一个 本质上需要建模高维傅里叶序列的complex 
%  tasks，就需要指数级的可学习参数。

% 还有一类MLP的变体它们不使用傅里叶基底，而是通过其它类型的函数或理论作为近似手段。比如 PWLNN models the target function using piecewise activation functions and can be viewed as an extension of ReLU-based MLP. It is adept at fitting complex nonlinear systems...,KAN 也属于这一领域，他使用b样条函数和大量的调参技巧去近似目标函数，这种方法通常需要更少的参数，同时可解释性也更好，在符合回归等问题上比经典MLP表现的更好。但是这些方法的表现主要集中在一些特定领域，对于更广泛的人工智能的作用还需进一步探索。

 % FAN: Other hybrid activation designs, such as those combining sine, cosine, and GELU functions, have also been proposed, showing consistent improvements across a wide range of AI tasks.

\begin{figure}[t]
    \centering
    \includegraphics[width=0.7\linewidth]{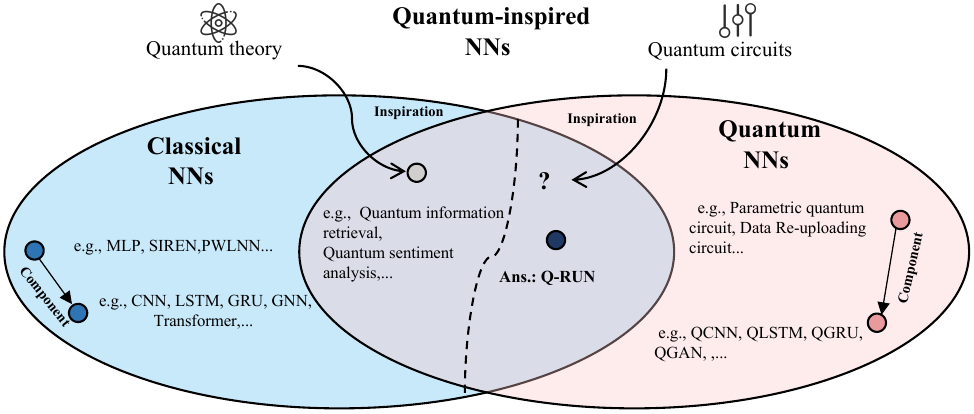}
    \caption{Q-RUN is situated within the related literature on variational quantum circuit–inspired neural networks.}
    \label{frw}
\end{figure}

\subsection{Quantum Neural Network}

% Similar to classical cases, QNNs also serve as foundational building blocks for diverse architectures in QML, such as quantum LSTMs \cite{chen2022quantum}, quantum topic models \cite{qiao2024quantum}, and quantum language models \cite{coecke2020foundations}. 

% Among these, data re-uploading circuits, an implementation of QNNs, have demonstrated strong capabilities in fitting truncated Fourier series \cite{yu2022power,schuld2021effect}. This enables QNNs to efficiently capture high-frequency components of natural continuous signals, such as images and sound waves \cite{zhao2024quantum}, and to perform time series prediction \cite{qiao2025quantum}, with fewer parameters.

% However, due to limitations of the NISQ era, these models are often challenging to deploy on real quantum hardware, and even when executable, noise interference makes it difficult to surpass the latest classical methods.

Similar to classical cases, QNNs also serve as foundational building blocks for diverse architectures in QML \cite{baek2023logarithmic}, such as recurrent neural networks \cite{chen2022quantum,li2024quantum}, graph neural networks \cite{yan2022towards,bai2025aegk}, and convolutional neural networks \cite{cong2019quantum}. These QNN-based models have been applied to reinforcement learning \cite{dong2008quantum}, adversarial learning \cite{lu2020quantum}, generative models \cite{anschuetz2023interpretable}, language modeling \cite{coecke2020foundations}, topic models \cite{qiao2024quantum}, and even fine-tuning large language models \cite{chen2024quanta,liu2024quantum,koike2025quantum}.
Furthermore, by reformulating certain tasks, QML can be applied to feature selection \cite{ferrari2022towards}, as well as to information retrieval and recommender systems \cite{ferrari2024using}.
However, due to limitations of the NISQ era, these models are often challenging to deploy on real quantum hardware, and even when executable, noise interference makes it difficult to surpass the latest classical methods.

A more practical near-term direction involves quantum-inspired neural networks, which apply mathematical principles from quantum mechanics to tackle specific tasks \cite{fan2024quantum,fan2024quantum2} without the stringent requirements of quantum computing, enabling easier implementation on classical hardware. For instance, mathematical tools from quantum theory have been used to simulate quantum circuits \cite{pan2022simulation}. 
In information retrieval, modeling user decisions as quantum collapse captures cognitive behaviors that classical networks struggle with \cite{bruza2015quantum,jiang2020quantum}, and similar ideas have been extended to recommendation systems \cite{han2024quantum,wang2019qpin}.
In language modeling, interpreting word ambiguity through many-body quantum properties improves tasks like question-answer matching \cite{zhang2018end,zhang2018quantum,zhang2022complex}, false information detection \cite{tian2020qsan}, and representing uncertainty in semantics and sentiment within natural language processing \cite{yan2021quantum, li2021quantum, shi2024pretrained,zhang2024quantum,qiao2024quantum2,shokrollahi2023intersectional}. Additionally, in network modeling, quantum-inspired techniques efficiently embed nodes into quantum spaces \cite{xiong2024node2ket}. However, these approaches are typically tailored to specific problems rather than serving as general-purpose methods applicable across diverse AI tasks.

In this paper, we propose Q-RUN, a quantum-inspired neural network with a key distinction: rather than drawing analogies from abstract quantum mechanical principles, Q-RUN is directly inspired by the concrete structure of quantum circuits. Accordingly, our approach lies in the intersection region on the right side of the dashed line in Figure \ref{frw}.

\section{Preliminaries}

Q-RUN relies on quantum computing and DRQC. This section briefly reviews the basics and highlights the theoretical strengths of DRQC.

\textit{Qubits,} the basic units of quantum circuits, can exist in a superposition of the basis states \( |0\rangle \) and \( |1\rangle \) \footnote{Quantum mechanics uses Dirac notation to represent vectors.}, forming the foundation of quantum computational frameworks.
This superposition is mathematically represented as \( |\phi\rangle = \alpha|0\rangle + \beta|1\rangle \), where \( |0\rangle = [1, 0]^\top \) and \( |1\rangle =[0, 1]^\top \) constitute an orthonormal basis. Here, $\alpha$ and $\beta$ are complex numbers that satisfy the normalization condition $|\alpha|^2+|\beta|^2=1$. Similarly, an $N$-qubit quantum circuit corresponds to a quantum state in a $2^N$-dimensional Hilbert space, written as $|\phi\rangle = \sum_{i=1}^{2^N}\alpha_i|x_i\rangle$, where $\sum_i |a_i|^2 = 1$ and $x_i \in \{0,1\}^N$ indicates a multi-bit string. The set $\{|x_i\rangle\}_{i=1}^{2^N}$ forms an orthonormal basis of the Hilbert space, constructed by tensor products \footnote{$\otimes$ denotes tensor product, e.g., $|10\rangle = [0,0,1,0]^\top=[0,1]^\top\otimes[1,0]^\top = |1\rangle \otimes |0\rangle$. } of $|0\rangle$ and $|1\rangle$. 

% The tensor product \( \otimes \) is a fundamental operation in quantum computing, used to combine quantum states or operators. Given two vectors \( \mathbf{a} \in \mathbb{C}^m \) and \( \mathbf{b} \in \mathbb{C}^n \), their tensor product is a vector in \( \mathbb{C}^{mn} \), defined as
% \[
% \mathbf{a} \otimes \mathbf{b} =
% \begin{bmatrix}
% a_1 \mathbf{b} \\
% a_2 \mathbf{b} \\
% \vdots \\
% a_m \mathbf{b}
% \end{bmatrix}.
% \]
% For matrices \( \mathbf{A} \in \mathbb{C}^{m \times n} \) and \( \mathbf{B} \in \mathbb{C}^{p \times q} \), the tensor product \( \mathbf{A} \otimes \mathbf{B} \) is a \( mp \times nq \) matrix:
% \[
% \mathbf{A} \otimes \mathbf{B} =
% \begin{bmatrix}
% a_{11}\mathbf{B} & \cdots & a_{1n}\mathbf{B} \\
% \vdots & \ddots & \vdots \\
% a_{m1}\mathbf{B} & \cdots & a_{mn}\mathbf{B}
% \end{bmatrix}.
% \]
% Tensor products are used to describe joint quantum states and composite system operators.

\textit{Quantum circuits,} built through ordered layers of quantum gates, utilize these gates to manipulate the states of qubits. For instance, the Pauli-$Z$ gate introduces a phase flip, while the CNOT gate creates entanglement between two qubits when applied to appropriate states. In this work, we focus on the Pauli-$Y$ gate, which represents a fundamental single-qubit operation. Its unitary matrix form is
\[
Y = 
\begin{bmatrix}
0 & -i \\
i & 0
\end{bmatrix}.
\]
The eigenvalues of the $Y$ gate are
$
\lambda_{1,2} = \pm 1.
$
These quantum gates are all regarded as unitary matrices, which evolve the qubit state according to
\[
|\phi\rangle' = U|\phi\rangle,
\]
where $U$ denotes the corresponding unitary operator (e.g., the Pauli-$Y$ gate).
Finally, measurements are performed to read out the results $x_i$. Due to the inherent randomness in quantum measurements, we often focus on the expected value of an observable, given by
\[
E(\boldsymbol{O}) = \langle \phi | \boldsymbol{O} | \phi \rangle,
\]
where $\langle \phi |$ is the bra vector, i.e., the conjugate transpose of $|\phi\rangle$, written as $\langle \phi | = (|\phi\rangle)^\dagger$. Here, $\boldsymbol{O}$ is a Hermitian operator (e.g., the Pauli-$Z$ gate).

\subsection{Data Re-uploading Quantum Circuit}

By further introducing rotational quantum gates (e.g., $R_y$), one can construct parameterized quantum circuits. We particularly focus on the $R_y$ gate since its matrix elements are entirely real, which makes it more amenable to implementation on classical devices. The unitary form of the $R_y$ gate is 
\begin{equation}\label{y}
R_y(\theta) = 
\begin{bmatrix}
\cos \tfrac{\theta}{2} & -\sin \tfrac{\theta}{2} \\
\sin \tfrac{\theta}{2} & \cos \tfrac{\theta}{2}
\end{bmatrix}.
\end{equation}
Based on these rotational gates, a parameterized quantum circuit can be constructed using a data-dependent unitary \( U(x) \) and a parameterized unitary \( U(W) \),  
\[
f(x) = \langle \phi |\boldsymbol{O}| \phi \rangle, 
\] 
with $| \phi \rangle = U(W)U(x)| 0 \rangle^{\otimes N}$.
Given a target function, the variational parameters \( W \) can be updated via gradient descent~\cite{mitarai2018quantum}, similar to the training process in neural networks. Therefore, parameterized quantum circuits serve as an implementation of QNN.

\begin{figure}[t]
    \centering
    \includegraphics[width=0.5\linewidth]{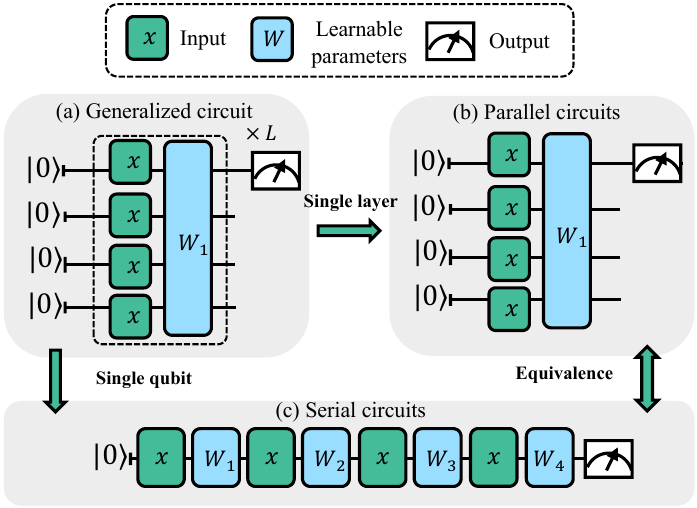}
\caption{Several variants of DRQC. (a) Generalized architecture. (b) Multi-qubit, single-layer encoding scheme. (c) Single-qubit, multi-layer encoding scheme. Schemes (b) and (c) are equivalent in expressive power.}    \label{p1}
\end{figure}

In this work, we focus on the data re-uploading quantum circuit (DRQC), in which encoding layers and parameter layers are alternately stacked, i.e., 
\[
| \phi \rangle = U(W_L) U(x)\ldots U(W_2) U(x) U(W_1) U(x) | 0 \rangle^{\otimes N}.
\]  
The structure of this circuit is illustrated in Figure~\ref{p1}(a). This architecture has an inductive bias capable of expressing Fourier series~\cite{schuld2021effect}.
Particularly, when the input data is a multi-dimensional vector \( \boldsymbol{x} \in \mathbb{R}^d \), and each element of \( \boldsymbol{x} \) is repeatedly encoded into \( N \) qubits and repeatedly stacked for \( L \) layers. This enables the circuit to exhibit:

\begin{claim}
The capability of DRQC to fit Fourier series grows exponentially with the size of the circuit under optimal conditions \cite{zhao2024quantum}. 
\end{claim}

\noindent\textbf{Remark.} This is because the circuit can eventually be derived in the following form:
$
        f(\boldsymbol{x}) = \langle \phi |\boldsymbol{O}| \phi \rangle = \sum_{\bm{\omega}\in\Omega} c_{\bm{\omega}} \mathrm{e}^{\bm{\omega} \cdot \boldsymbol{x}},
$
where the frequency spectrum $\Omega =\{-NL, \ldots, 0,\ldots ,NL\}^{d}$. 
This illustrates the inductive bias of DRQC towards representing Fourier series, thereby facilitating the modeling of high-frequency components. In particular, its expressivity is exponentially more efficient. Because a DRQC can represent a Fourier series with spectrum size \((2NL + 1)^d\) using only \( \mathcal{O}(d N L) \) parameters, while a conventional Fourier series typically requires \( \mathcal{O}((NL)^d) \) parameters.
It is unfortunate, however, that a $d$-dimensional input theoretically requires a circuit with $(d\cdot N)$ qubits. In practice, modern AI systems often involve data with hundreds or thousands of dimensions, yet current quantum computers cannot reliably handle that many qubits.

\section{Methodology}

Given the aforementioned advantages and limitations of quantum models, this section proposes Q-RUN. We first present a rigorous theoretical formulation and analyze its expressive power, demonstrating its unique capability in modeling Fourier series. Subsequently, we introduce a relaxed implementation that enables Q-RUN to run efficiently on classical hardware without requiring quantum devices.

%鉴于上述量子模型的优势与局限，本节将构建一个受Quantum-inspired Data Re-uploading Network，首先给出一个严格的理论实现以及表达能力的分析它展示它独特的建模能力，然后给出一个放松的实现方案使Q-RUN能高效的运行在经典计算机而无需量子硬件

% Given the above limitations, this section proposes a quantum-inspired data re-uploading neural layer (Q-RUN), as shown in Fig. \ref{f3}. This method transfers the Fourier-series-like expressivity of quantum models to classical architectures, enabling more scalable modeling capabilities. Unlike QML, Q-RUN is not constrained by quantum hardware. It also differs from previous quantum-inspired approaches that rely on loose analogies to quantum principles. Instead, Q-RUN is more directly grounded in the fundamental structure and rules of quantum computation, providing a more rigorous theoretical foundation.

\subsection{Theoretical Formulation of Q-RUN}

To incorporate the expressive power of DRQC into neural networks, a natural idea is to replace the fully connected layers with a DRQC-based computation process. 

Let \( \boldsymbol{x} \in \mathbb{R}^d \) be the input vector and \( \boldsymbol{w} \in \mathbb{R}^{n} \) be a learnable parameter vector that applies element-wise scaling across \( n \) repeated uploads. We define the data re-uploading operator as:
\begin{equation}\label{e1}
S(\boldsymbol{x}) := \bigotimes_{i=1}^d \underbrace{Ry(w_{1} x_i) \otimes Ry(w_{2} x_i) \otimes \cdots \otimes Ry(w_{n}x_i)}_{n\ \text{times}}
\end{equation}
where \( w_j \) denotes the \(j\)-th element of \(\boldsymbol{w}\). The single-qubit rotation gate \( R_y \) is defined by Eq.~(\ref{y}). This construction repeatedly encodes each dimension of \( \boldsymbol{x} \) into a subsystem of \( n \) qubits, thereby realizing the core idea of data re-uploading in DRQC. It should be noted that this corresponds to the structure illustrated in Figure~\ref{p1}(b), where a single-layer, multi-qubit circuit is employed to implement the DRQC. The rationale for adopting this design lies in efficiency: both multi-layer, multi-qubit and multi-layer, single-qubit circuits require sequential execution layer by layer, which is inefficient in Q-RUN \cite{wang2025predictive}. In contrast, using a single-layer, multi-qubit structure to extend the number of data re-uploads allows for more parallelization and higher efficiency. This strategy is generally not recommended in QNNs because increasing the number of qubits is much more costly than adding layers \cite{nguyen2022evaluation}. However, in Q-RUN, increasing the number of qubits corresponds to increasing the number of neurons. This operation is relatively inexpensive in neural networks, which makes it particularly suitable for efficient execution in Q-RUN.

Following DRQC paradigm, the qubits first undergo evolution and are subsequently measured. In Q-RUN, we therefore first apply the data re-uploading encoding \( S(\boldsymbol{x}) \) to the initial qubits \( \ket{0}^{\otimes nd} \), followed by unitary evolution \( U \), and finally compute the expectation with respect to an observable \( \boldsymbol{O} \):
\begin{equation}\label{e2}
f(\boldsymbol{x}) = \bra{0}^{\otimes nd} \, S(\boldsymbol{x})^\dagger U(W)^\dagger \boldsymbol{O} U(W)S(\boldsymbol{x}) \ket{0}^{\otimes nd}
\end{equation}
Here, \( W \) is a learnable variational parameter controlling the unitary evolution, whose size grows polynomially with the number of qubits in QML. For ease of reasoning, we equivalently treat the combination of unitary evolution and measurement as a single learnable observable \( \hat{\boldsymbol{O}} \), yielding:
\begin{equation}\label{e3}
f(\boldsymbol{x}) = \bra{0}^{\otimes nd} \, S(\boldsymbol{x})^\dagger \, \hat{\boldsymbol{O}} \, S(\boldsymbol{x}) \ket{0}^{\otimes nd}
\end{equation}

% Following the DRQC paradigm, the qubits will undergo evolution and measurement. 因此在 Q-RUN中，对于一个初始的状态，首先将S(x)作用到初始的量子比特|0》，然后经过U(W)的酉矩阵进行演化，最后以O作为观测量，求Q-RUN输出的期望：
% \begin{equation}\label{e2}
% f(\boldsymbol{x}) := \bra{0}^{\otimes nd} \, S(\boldsymbol{h})^\top \, \boldsymbol{O} \, S(\boldsymbol{h}) \ket{0}^{\otimes nd}
% \end{equation}
% 其中W是可学习的变分参数，用于控制酉矩阵对量子比特的状态进行挑战，受量子计算的规则的影响，它的数量级大概是nq的多项式级别。
% For more convenient reasoning, we equivalently treat the combined evolution and measurement 为一个整体，成为learnable observable.：
% \begin{equation}\label{e2}
% f(\boldsymbol{x}) := \bra{0}^{\otimes nd} \, S(\boldsymbol{h})^\top \, \boldsymbol{O}^hat \, S(\boldsymbol{h}) \ket{0}^{\otimes nd}
% \end{equation}

Building on DRQC expressivity analysis \cite{schuld2021effect,zhao2024quantum}, we prove the rigorous of Q-RUN has strong expressive power for Fourier series.

% \begin{claim}
% Q-RUN can represent a truncated Fourier expansion with up to \( 3^{n d} \) distinct frequency components using only \( O(2^{n d} \cdot r) \) parameters, where \( r \) is the rank of the parameter matrix.
% \end{claim}

% In classical Fourier series parameterizations, representing all \( 3^{n d} \) frequency terms requires \( O(3^{n d}) \) parameters. In contrast, Q-RUN only needs \( O(2^{n d} \cdot r) \). When \( r \) reaches its maximal value (full rank), the symmetric matrix \( M \) can span the entire space of real symmetric matrices, effectively covering all possible Fourier components.

% However, this would eliminate any compression advantage, reducing Q-RUN to a traditional Fourier model. Fortunately, numerous works on low-rank approximations have shown that most real-world problems do not require full spectrum coverage. A suitable low-rank representation is often sufficient to achieve strong fitting and generalization, thanks to the implicit structure and redundancy in most practical data distributions.

%通过可训练参数w和A,Q-RUN 能够表达一个频谱大小为3^dn的多维傅里叶序列，只需要O(2^dn.r)的参数，在最好的情况下。

\begin{claim}\label{c2}
Under optimal conditions, Q-RUN is capable of expressing a $d$-dimensional truncated Fourier series with a frequency spectrum containing up to $3^{dn}$ components.
\end{claim}

\begin{proof}[Sketch of Proof]
Since Eq.~(\ref{e1}) can be diagonalized as \( S(\boldsymbol{x}) = V^\dagger H V \), where \( H \) is diagonal matrix and \( V \) is unitary matrix. The left matrix \( V^\dagger \) can be absorbed into the observable \( \hat{\boldsymbol{O}} \), while the right matrix \( V \) can be merged into the initial state \( \ket{0}^{\otimes nd} \), yielding $\ket{\Gamma} := \sum_{\boldsymbol{j}} \gamma_{\boldsymbol{j}} \ket{\boldsymbol{j}}$.
Using this property, the Eq. (\ref{e3}) can be derived as:
\[
f(\boldsymbol{x}) = \sum_{\boldsymbol{j},\boldsymbol{k}} \gamma_{\boldsymbol{j}}^* \gamma_{\boldsymbol{k}}\hat{\boldsymbol{O}}_{\boldsymbol{j},\boldsymbol{k}} \, e^{i \boldsymbol{x} \cdot (\boldsymbol{\lambda_k} - \boldsymbol{\lambda_j)}},
\]
where \( \boldsymbol{j}, \boldsymbol{k} \in \{2^{n}\}^{d} \) denote multi-indices over a \( d \)-dimensional system with \( n \) qubits per dimension. The corresponding eigenvalues \( \boldsymbol{\lambda}_{\boldsymbol{j}} \in \left\{ \pm \frac{w_{1}}{2}, \ldots, \pm \frac{w_{n}}{2} \right\}^{d} \) are obtained by indexing the diagonal matrix \( H \) with multi-indices. The set of all possible values of the difference \( \boldsymbol{\lambda}_{\boldsymbol{k}} - \boldsymbol{\lambda}_{\boldsymbol{j}} \) determines the frequency spectrum that Q-RUN is capable of expressing.
The introduction of a learnable parameter vector $\boldsymbol{w}$ ensures that the resulting frequencies are non-degenerate, avoiding repeated values in the spectrum. In the ideal case, a well-trained \( \boldsymbol{w} \) can lead to the maximal number of frequency components, satisfying
$
\left| \left\{ \boldsymbol{\lambda}_{\boldsymbol{k}} - \boldsymbol{\lambda}_{\boldsymbol{j}} \right\} \right| = 3^{d n}.
$
For a complete and detailed proof, please refer to Appendix~\ref{app1}.
\end{proof}

This shows that Q-RUN successfully introduces the inductive bias of quantum models toward fitting Fourier series, allowing it to address the challenge of approximating high-frequency functions. On the other hand, due to the inclusion of $\boldsymbol{w}$, Q-RUN can cover a frequency spectrum of up to \(3^{nd}\), which typically requires only a polynomial number of variational parameters $W$. This indicates that Q-RUN is capable of modeling a richer frequency spectrum with fewer parameters. It should be noted that Claim~\ref{c2} can be regarded as an extension of the results in \cite{schuld2021effect} and \cite{zhao2024quantum}. In particular, compared to \cite{schuld2021effect}, Q-RUN further extends the frequency spectrum by introducing the learnable parameter vector, thereby enhancing the Fourier series representation capability of the model. In comparison with \cite{zhao2024quantum}, Q-RUN derives a single-layer, multi-qubit structure whose capacity can be extended simply by increasing the number of qubits. This structure can be naturally embedded within classical neural network architectures.

\subsection{Relaxed Implementation of Q-RUN}

    While Q-RUN is theoretically promising, it involves extensive tensor products (e.g., Eq.~(\ref{e1})) and large matrix multiplications (e.g., Eq.~(\ref{e3})), which become infeasible for high-dimensional tasks on classical hardware due to the exponential growth of the Hilbert space. To make the approach scalable, we relax the strict quantum operations and approximate them using efficient structures with similar computational efficiency to fully connected layers. The complete quantum formulation is left to future quantum hardware advancements. The relaxed version of Q-RUN abandons tensor product operations and adopts element-wise parameter sharing. As illustrated in Figure~\ref{fig2}, it consists of two modules: Data Re-uploading and Element-wise Observable module.

\begin{figure}[t]
    \centering
    \includegraphics[width=0.65\linewidth]{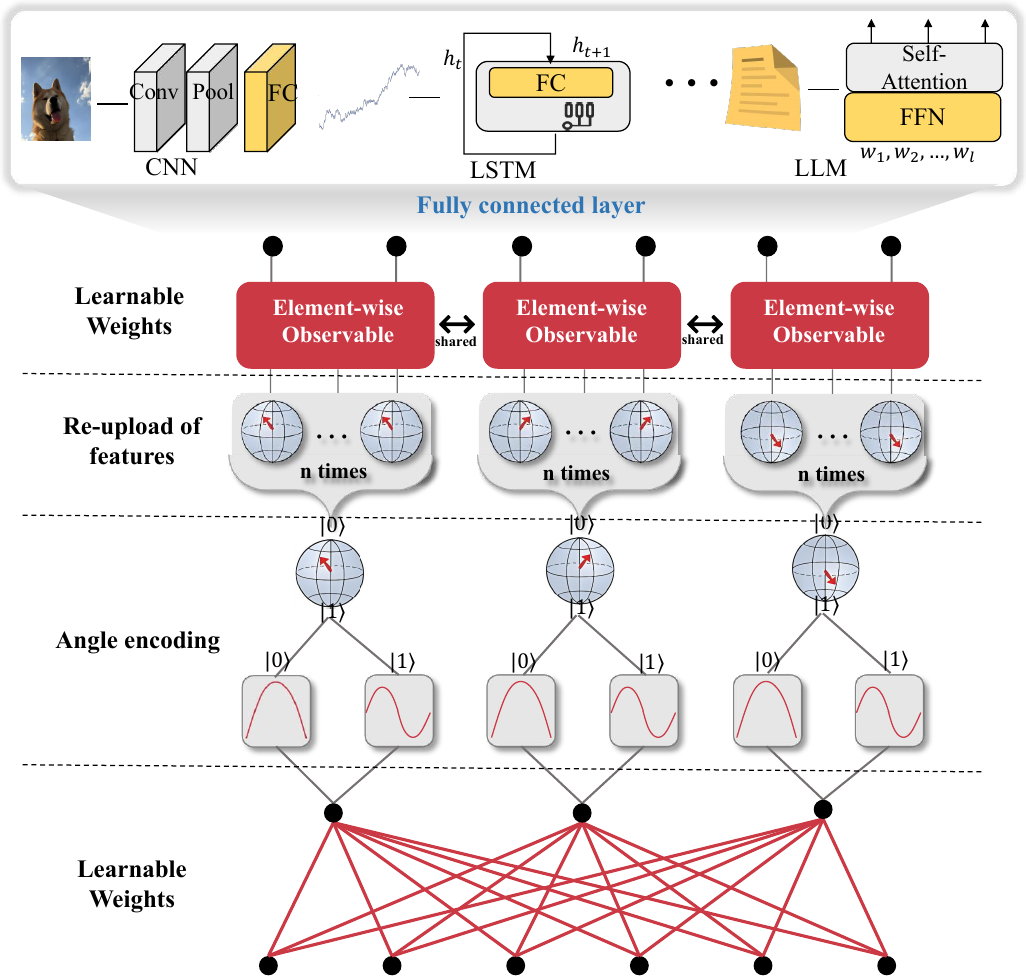}
\caption{The Q-RUN layer under relaxed implementation can replace fully connected layers of diverse architectures.}

    \label{fig2}
\end{figure}

% Q-RUN的理论形式具有希望，但是遗憾的其中涉及大量的张量积运算与大矩阵乘法。在小规模的隐含层尚且可以实现，但对于一个复杂的人工智能任务的隐含层可能要处理成百上千的隐含特征，这将需要构建 2的指数级别的空间, 这在当今最厉害的计算机上也不可能实现。为此我们必须对Q-RUN中严格的量子计算过程进行放松，以匹配当前全连连接层的效率。而对于严格版本的Q-RUN以待未来拥有成熟的量子计算机来实现。

% Inspired by the design of variational quantum circuits (VQCs), we propose a novel quantum-inspired module, termed the \emph{Quantum-Inspired Feature Angle Network (QIFAN)}. The core idea is to simulate classical analogues of quantum angle encoding, data re-uploading, and measurement-induced nonlinearity using purely classical neural modules. The QIFAN layer is implemented using a composition of linear projections, trigonometric functions, feature stacking, and a small multi-layer perceptron (MLP). This structure implicitly constructs a nontrivial Fourier feature basis and enhances the frequency expressivity of the model.

% \paragraph{Data Re-uploading Module.}
% Let \( \boldsymbol{x} \in \mathbb{R}^{d} \) denote the input. The Q-RUN  first applies a linear projection, 将d维数数据 reduced by a factor of {} ，
% \[
% x' = \boldsymbol{W} \boldsymbol{x} + b,
% \]
% 这个操作是为了保证Q-RUN中间变量不会占用过多的空间。通常情况下可设置为2，因为角度编码同时需要处理保存sin和cos的数据。该值可随着重上传次数n的增加而调整。然后针对x'中的每一个元素，我们仍然应用RY门并作用在初始的0态上，这个过程可以直接表达为：
% |hij>=cos{}{wi*xj}+sin(wi*xj).
% 其中\boldsymbol{w}\in R^n 可以是一个全层共享的参数向量，用于实现和前文参数矩阵类似的作用。此时，我们不在进行复杂的进行张量积运算，而将这些数据重复编码后的量子态直接拼接：
% h(j)=[|h1j>;|h2j>;...;|hnj>]
% 对于全部的元素，我们将得到一个特征张量h=[h(1),h(2)，。。。，,h(d/z)]
% 上述特征编码过程类似于一些经典傅里叶神经网络使用sincos函数对中间特征进行激活的行为，但本质的不同之处在与Q-RUN还会同样的元素进行多次编码取，而这样的理论依据是受到DRQC原理启发的。

\subsubsection{Data Re-uploading Module}

Let \( \boldsymbol{x} \in \mathbb{R}^{d} \) denote the input. Q-RUN first applies a linear projection \( f_0 \) to reduce the input dimension by a factor of \( \alpha \), resulting in:
\begin{equation}\label{e4}
\boldsymbol{x}' = f_0(\boldsymbol{x})=\boldsymbol{W_0} \boldsymbol{x} + \boldsymbol{b_0},
\end{equation}
where \( \boldsymbol{W_0} \in \mathbb{R}^{d/\alpha \times d} \) and \( \boldsymbol{b_0} \in \mathbb{R}^{d/\alpha} \). This step controls intermediate memory cost, preventing excessive growth of re-uploaded representations. The value of \(\alpha\) can be adjusted based on the number of re-uploads \(n\). Typically, we set \(\alpha = 2\) to halve the input dimension, compensating for the subsequent increase in feature size.

For each element \( x'_j \) in \(\boldsymbol{x}'\), we similarly apply an $Ry$ gate to the initial state \(\lvert 0 \rangle\). Without derivation, the result of the $i$-th encoding is given directly as:
\begin{equation}
\lvert h_{i,j} \rangle = \cos(w_i x'_j) \lvert 0 \rangle + \sin(w_i x'_j) \lvert 1 \rangle = Ry(w_i x'_j) \lvert 0 \rangle,
\end{equation}
where \(\boldsymbol{w} \in \mathbb{R}^n\) is a parameter vector shared across all elements. Instead of computing tensor products, we concatenate the encoding repeated \( n \) times for the \( j \)-th feature dimension:
\begin{equation}
h^{(j)} := \underbrace{\left[\lvert h_{1,j} \rangle; \lvert h_{2,j} \rangle; \ldots; \lvert h_{n,j} \rangle\right]}_{n\ \text{times}}.
\end{equation}
While reminiscent of the sinusoidal activations used in Fourier-based neural networks \cite{sitzmann2020implicit,dong2024fan}, this module fundamentally differs by re-uploading each element multiple times and subsequently processing them with an observable module. This construction is precisely what is inspired by DRQC.
% and form the full representation as a stacked tensor:
% \begin{equation}
% \boldsymbol{h} = \left[h^{(1)}, h^{(2)}, \ldots, h^{(d/\alpha)}\right].
% \end{equation}

% 这样我们就将问题限制在元素级别上,并且you

\subsubsection{Element-wise Observable Module}
To efficiently execute Q-RUN, we relax the strict reliance on quantum measurements in their theoretical form and instead introduce a lightweight MLP to approximate the measurement process. According to the universal approximation theorem \cite{hornik1991approximation}, an MLP with sufficient expressive power can approximate any measurable function. Therefore, we expect MLP to approximate the evolution and measurement behavior of a small quantum system. 

In Figure~\ref{F_A_2}, we illustrate the internal structure of the Element-wise Observable Module. This module is implemented as a three-layer element-wise MLP with an activation function. Specifically, $f_1$ approximates the tensor product, $f_2$ simulates the unitary evolution, and $f_3$ estimates the measurement expectation. For each input $h^{(j)} \in \mathbb{R}^d$, the forward process is given by:
\begin{equation}
f_{\hat{\boldsymbol{O}}}(h^{(j)})
:= f_3 \circ \sigma \circ f_2 \circ \sigma \circ f_1 (h^{(j)}),
\end{equation}
Each layer is a linear transformation with learnable parameters 
$\boldsymbol{W}_i$ and $\boldsymbol{b}_i$, followed by an activation function $\sigma(\cdot)$. 
The layer $f_2$ uses weight matrices $\boldsymbol{W}_2 \in \mathbb{R}^{m \times m}$, where $m$ is the hidden dimension. 
This dimension determines the expressive capacity of the Element-wise Observable Module, but it is usually kept small for computational efficiency. 
The output dimension of $f_3$ can be adjusted to control the Q-RUN output. 
This is analogous to performing a multi-basis measurement in a QNN by adjusting the number of bases. 
For example, setting the output dimension of $f_3$ to $\alpha$ ensures that the Q-RUN output matches its input dimension $d$.
Moreover, as this module depends only on the re-uploading times $n$ and is shared across all input dimensions, the computational cost of Q-RUN layers does not increase significantly.

\begin{figure}[t]
    \centering
    \includegraphics[width=0.6\linewidth]{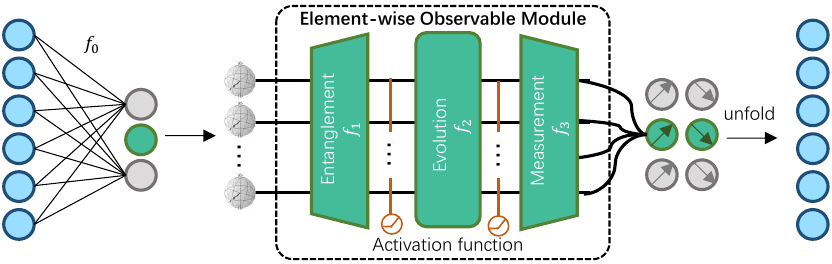}
    \caption{The internal structure of the Element-wise Observable Module.}
    \label{F_A_2}
\end{figure}

\begin{table*}[t]
\centering
\caption{
Comparison of parameter count and FLOPs between a standard MLP layer and Q-RUN layer. Here, \(d_{\text{in}}\) and \(d_{\text{out}}\) denote the input and output dimensions, respectively. \(\text{FLOPs}_{\text{act}}\) represents the FLOPs of the activation function, usually Tanh, while \(\text{FLOPs}_{\sin/\cos}\) refers to the FLOPs associated with the sinusoidal encoding elements in the data re-uploading process.}
\begin{tabular}{@{}lcc@{}}
\toprule
\textbf{Property} & \textbf{MLP Layer} & \textbf{Q-RUN Layer} \\
\midrule
Number of Parameters
&
$d_{\text{in}} d_{\text{out}}$
&
$\frac{d_{\mathrm{in}}  d_{\mathrm{out}}}{\alpha} + n (1 + 2m) + m^2 + m \alpha$\\
\midrule
FLOPs
&
$2  (d_{\mathrm{in}}  d_{\mathrm{out}}) + \text{FLOPs}_{\text{act}}  d_{\text{out}}$
&
$2 \frac{d_{\mathrm{out}}}{\alpha} \left( d_{\mathrm{in}} + n \mathrm{FLOPs}_{\sin/\cos} + m^{2} + m \left( 2 n + \alpha + \mathrm{FLOPs}_{\mathrm{act}} \right) \right)
$ \\
\bottomrule
\end{tabular}
\label{tab:mlp_qrun_comparison}

\end{table*}

% where each layer is parameterized by weights \( W_1 \in \mathbb{R}^{n \times d_{\mathrm{mlp}}} \), \( W_2 \in \mathbb{R}^{d_{\mathrm{mlp}} \times d_{\mathrm{mlp}}} \), and \( W_3 \in \mathbb{R}^{d_{\mathrm{mlp}} \times \alpha} \), respectively, with the activation function \( \sigma(\cdot) \) applied between layers (bias term is omitted for simplicity). The layer \( f_3 \) approximates a multi-base measurement with basis number \(\alpha\), ensuring that the output dimension of Q-RUN matches its input dimension. Moreover, since this MLP operates solely based on the re-uploading times \( n \) and is shared across all dimensions, it introduces no significant additional parameters or computational overhead.

Finally, the output of Q-RUN layers is given by:
\begin{equation}
f(\boldsymbol{x}) := [f_{\hat{\boldsymbol{O}}}(h^{(1)});f_{\hat{\boldsymbol{O}}}(h^{(2)});...;f_{\hat{\boldsymbol{O}}}(h^{(d/\alpha)})]^\top.
\end{equation}
Compared to the theoretical formulation of Q-RUN, this design reduces computational cost while preserving strong Fourier expressiveness. 
Relative to a standard MLP, Q-RUN typically uses fewer parameters due to the unique inductive bias introduced by the quantum-inspired model and can potentially achieve lower fitting error. In Table~\ref{tab:mlp_qrun_comparison}, we compare the parameter count and floating-point operations (FLOPs) of Q-RUN with those of a standard MLP. 
Although Q-RUN employs a more complex structure, the number of parameters can be maintained at a comparable or even smaller scale than that of a standard MLP by properly configuring the hyperparameters (such as $\alpha$, $n$, and $m$). Moreover, while Q-RUN involves more FLOPs, its modular design and element-wise observable computation enable high parallelism, making it significantly more practical than the theoretical form.

\section{Experiments}

Theoretically, Q-RUN’s Fourier series expressivity enables it to excel in complex data modeling tasks and holds promise for broader AI. To assess this potential, we conduct experiments guided by the following research questions:

\begin{itemize}
    \item \textbf{RQ1:} Can Q-RUN serve as an effective approximation of DRQC?
    \item \textbf{RQ2:} Can the Fourier series expressivity of Q-RUN be effectively applied to \textit{data modeling tasks}?
    \item \textbf{RQ3:} Can Q-RUN be generalized to broader AI tasks, such as \textit{predictive modeling tasks}?
    \item \textbf{RQ4:} How do the key design factors of Q-RUN affect its performance?
\end{itemize}

The data modeling experiments were conducted using a PyTorch implementation on an NVIDIA RTX 3060 GPU, while the prediction modeling tasks were performed on an NVIDIA RTX 4090 GPU, since the latter involves larger models. In contrast, data modeling tasks are relatively lightweight and can be efficiently executed on the RTX 3060. To ensure fair validation, the scale of Q-RUN was controlled to be comparable to or smaller than that of the classical baselines, eliminating interference from model size differences.

\begin{figure*}
    \centering
    \includegraphics[width=1\linewidth]{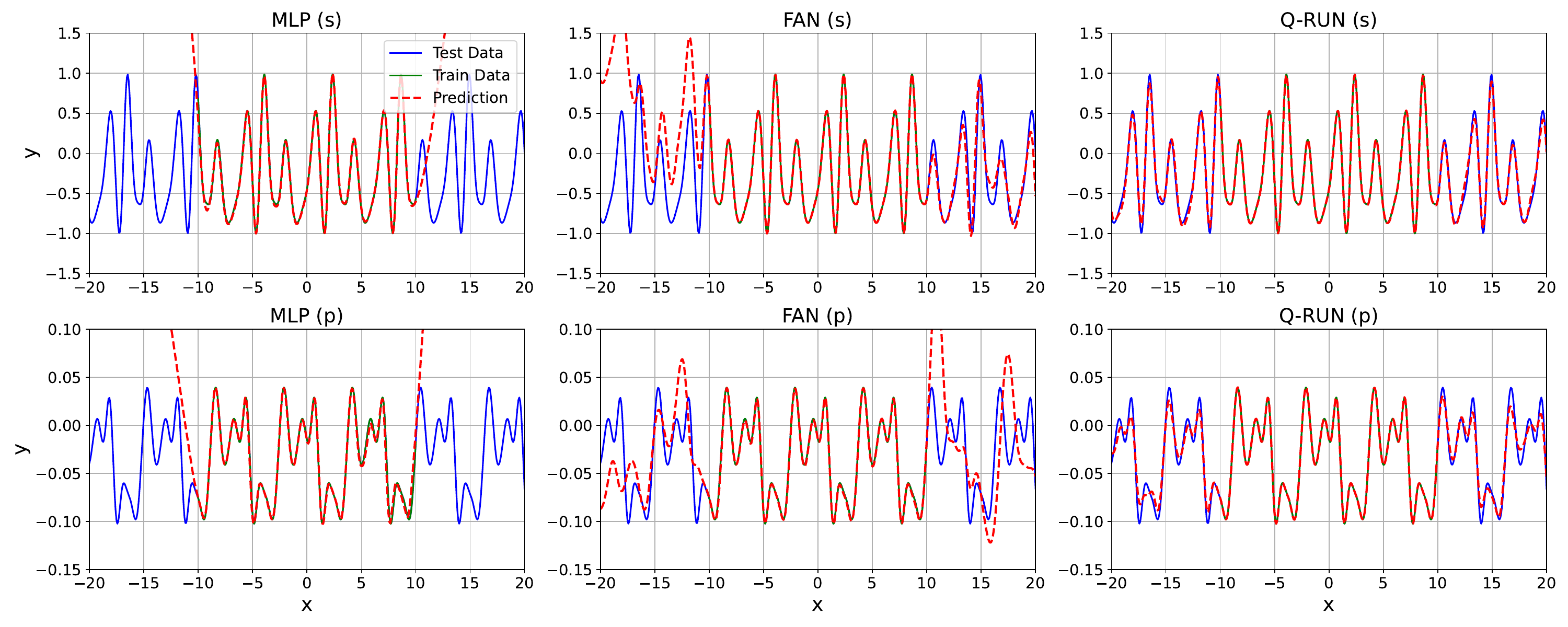}
    \caption{Performance of Q-RUN, FAN, and MLP in approximating DRQC. (s) denotes single-qubit multi-layer DRQC, and (p) denotes multi-qubit single-layer DRQC.}
    \label{ef1}
\end{figure*}

\subsection{RQ1: Experiments on DRQC Approximation}

As Q-RUN is a classical neural network layer inspired by DRQC, it is natural to ask whether Q-RUN can effectively approximate the behavior of DRQC. To investigate this, we randomly initialized a single-qubit DRQC with data re-uploaded 8 times, using the structure illustrated in Fig.~\ref{p1}(c), and an 8-qubit DRQC with a single-layer multi-qubit structure, as illustrated in Fig.~\ref{p1}(b). We uniformly sampled 1,000 input points from the interval $[-20, 20]$, and obtained their corresponding outputs by executing the DRQC on the Pennylane quantum simulator \cite{bergholm2018pennylane}, resulting in a dataset ${(x_i, y_i)}$. The subset with $x \in [-10, 10]$ was used for training, while the remainder served as the test set for evaluating generalization. We compared against MLP and Fourier Analysis Networks (FAN), where FAN enhances MLPs with the capacity to model periodic patterns through Fourier series \cite{dong2024fan}.

The results are presented in Figure~\ref{ef1}. As expected, MLP performs poorly outside the training domain, primarily due to its inability to capture high-frequency components. FAN demonstrates improved periodic modeling, but its approximation of DRQC remains limited. This result was unexpected, suggesting that the periodic functions from DRQC, while appearing simple, may actually reflect complex high-frequency structures originating from quantum phenomena such as entanglement and superposition, which makes them challenging for classical models to approximate. In contrast, Q-RUN achieves the best performance, indicating that it successfully incorporates key DRQC-inspired properties and exhibits stronger capability in modeling complex, high-frequency quantum functions. This insight inspires AI for Quantum, and Q-RUN may emerge as a promising candidate for quantum circuit simulation \cite{jones2019quest}.

\subsection{RQ2: Experiments on Data Modeling Tasks}

In this subsection, we consider three fundamental data modeling tasks, namely implicit representation, density estimation, and energy modeling, which cover diverse modalities and demonstrate broad generality.

\subsubsection{Implicit Representation}

Implicit Representation (IR) aims to learn a continuous function
$f_\theta: \mathbb{R}^d \rightarrow \mathbb{R}$ that maps spatial coordinates
to corresponding signal values. This formulation provides a compact and continuous description of data, enabling the recovery of high-resolution signals from arbitrary query points in the coordinate space. In the literature, IR tasks particularly require the ability to accurately fit high-frequency functions \cite{sitzmann2020implicit,zhao2024quantum}.

\begin{table}[!t]
\renewcommand{\arraystretch}{1}
\setlength{\tabcolsep}{12pt}
\caption{Implicit Representation results measured by Mean Squared Error (MSE, $\times 10^{-3}$). \#Params indicates the number of learnable parameters in the model. Lower values indicate better performance. The best result in each column is highlighted in gray.}
\centering
\begin{tabular}{c c c c c c c}
\toprule
Method & \#Params & Cello & \#Params & Astronaut & Camera & Coffee \\
\midrule
Tanh     & 831 & $14$ & 841 & $20$  & $5.8$ & $14.8$ \\
ReLU     & 831 & $6.8$ & 841 & $9.9$  & $2.7$ & $4.2$ \\
PWLNN    & 945 & $6.0$ & 945 & $9.4$  & $1.0$ & $2.8$ \\
FNO      & 689 & $5.4$ & 744 & $9.9$  & $4.1$ & $3.9$ \\
RFF      & 791 & $6.0$ & 791 & $5.1$  & $1.9$ & $4.9$ \\
SIREN    & 691 & $5.5$ & 701 & $9.0$  & $1.5$ & $2.3$ \\
FAN      & 691 & $6.4$ & 711 & $9.7$  & $1.4$ & $2.3$ \\
KAN      & \cellcolor{gray!20}{646} & \cellcolor{gray!20}$5.3$ & 672 & $12.0$ & $2.4$ & $2.8$ \\
QIREN    & 649 & $5.5$ & 657 & \cellcolor{gray!20}$4.0$ & $1.1$ & $1.5$ \\
\midrule
Q-RUN    & \cellcolor{gray!20}646 & \cellcolor{gray!20}$5.3$ & \cellcolor{gray!20}649 & $5.5$ & \cellcolor{gray!20}$0.6$ & \cellcolor{gray!20}$1.2$ \\
\bottomrule
\end{tabular}
\label{tb2}
\end{table}

\textit{Datasets.} We consider datasets consisting of one-dimensional audio signals and two-dimensional images, which are standard benchmarks widely used in IR studies. 
\begin{itemize}
    \item \textbf{Cello:} This dataset is derived from a short excerpt of a movement in Bach’s Cello Suites. The raw waveform is uniformly sampled at 1000 time points, with amplitudes normalized to the range $[-1, 1]$.
    \item \textbf{Astronaut, Camera, and Coffee:} These are three widely used grayscale images from the computer vision community \cite{van2014scikit}. Each image is cropped and downsampled to $32 \times 32$ pixels following the QIREN setting \cite{zhao2024quantum}, which imposes extreme parameter constraints such that the number of learnable parameters is significantly smaller than the number of pixels.
\end{itemize}

\textit{Baselines.} We compare Q-RUN against several baselines, categorized into two groups.
\begin{itemize}
    \item The first group comprises Fourier-based models:
    \begin{itemize}
        \item \textbf{SIREN} \cite{sitzmann2020implicit}: Implicit MLPs using sinusoidal activations to learn continuous signals with fine details.
        \item \textbf{RFF} \cite{tancik2020fourier}: RFF uses random Fourier features for input encoding to improve generalization and convergence.
        \item \textbf{FAN} \cite{dong2024fan}: Fourier Analysis Network designed to enhance learning across diverse AI tasks by leveraging Fourier representations, and can be considered a state-of-the-art approach.
        \item \textbf{FNO} \cite{kovachki2021neural}: The Fourier Neural Operator (FNO) is a neural network architecture capable of learning solution operators for partial differential equations (PDEs) directly in the frequency domain.
        \item \textbf{QIREN} \cite{zhao2024quantum}: A quantum-based implicit neural representation model employing data re-uploading circuits, which can be regarded as a state-of-the-art approach.
    \end{itemize}

    \item The second group includes models built upon other bases:
    \begin{itemize}
        \item \textbf{MLPs (Tanh/ReLU)}: Standard multilayer perceptrons with classical activation functions.
        \item \textbf{PWLNN} \cite{tao2022piecewise}: Piecewise linear neural networks approximate functions by dividing the input space into regions, within each of which the function is represented by a linear mapping.
        \item \textbf{KANs} \cite{liu2024kan}: Kolmogorov–Arnold Networks are neural architectures inspired by the Kolmogorov–Arnold representation theorem, where each network edge is modeled by a learnable univariate function (e.g., B‑spline function), rather than a fixed linear weight.
    \end{itemize}
\end{itemize}

\textit{Results.}
Experimental results for IR are shown in Table~\ref{tb2}. Among standard MLPs, Tanh exhibits the worst performance, while ReLU performs slightly better. PWLNN can be regarded as an enhanced version of ReLU, which explains its slightly better performance, although it also requires more learnable parameters. KANs achieve strong performance on one-dimensional signal tasks such as audio, but show limited gains on two-dimensional image signals. In contrast, methods incorporating Fourier features generally outperform the above approaches. Specifically, FNO achieves the second-best results on audio signal representation, while RFF handles high-dimensional IR tasks more effectively, albeit with an increased number of parameters. SIREN excels on IR across both 1D and 2D signals, and FAN performs similarly to SIREN. QIREN achieves good results on IR for lower-dimensional signals but struggles to scale to larger model sizes due to quantum hardware constraints. In contrast, Q-RUN model achieves nearly the best results overall, both in terms of the number of learnable parameters and the approximation error of the represented signals. We attribute this performance to the strong Fourier series fitting capability of Q-RUN.

% \begin{table}[t]
% \renewcommand{\arraystretch}{1}
% \setlength{\tabcolsep}{12pt}
% \caption{Implicit Representation results measured by Mean Squared Error (MSE, $\times 10^{-3}$). \#Params indicates the number of learnable parameters in the model. Lower values indicate better performance. The best result in each column is highlighted in gray.}
% \centering
% \begin{tabular}{c c c c c c c}
% \toprule
% Method & \#Params & Cello & \#Params & Astronaut & Camera & Coffee \\
% \midrule
% Tanh     & 831 & $14$ & 841 & $20$  & $5.8$ & $14.8$ \\
% ReLU     & 831 & $6.8$ & 841 & $9.9$  & $2.7$ & $4.2$ \\
% PWLNN    & 945 & $6.0$ & 945 & $9.4$  & $1.0$ & $2.8$ \\
% FNO      & 689 & $5.4$ & 744 & $9.9$  & $4.1$ & $3.9$ \\
% RFF      & 791 & $6.0$ & 791 & $5.1$  & $1.9$ & $4.9$ \\
% SIREN    & 691 & $5.5$ & 701 & $9.0$  & $1.5$ & $2.3$ \\
% FAN      & 691 & $6.4$ & 711 & $9.7$  & $1.4$ & $2.3$ \\
% KAN      & \cellcolor{gray!20}{646} & \cellcolor{gray!20}$5.3$ & 672 & $12.0$ & $2.4$ & $2.8$ \\
% QIREN    & 649 & $5.5$ & 657 & \cellcolor{gray!20}$4.0$ & $1.1$ & $1.5$ \\
% \midrule
% Q-RUN    & \cellcolor{gray!20}646 & \cellcolor{gray!20}$5.3$ & \cellcolor{gray!20}649 & $5.5$ & \cellcolor{gray!20}$0.6$ & \cellcolor{gray!20}$1.2$ \\
% \bottomrule
% \end{tabular}
% \label{tb2}
% \end{table}

\subsubsection{Density Estimation}

Density Estimation (DE) aims to model a normalized probability density function $P_\theta(X)$. 
This task requires a specialized loss function. Specifically, the task predicts an unnormalized density function \( q(x) \), which is normalized by dividing by its integral:
\[
\hat{q}(x) = \frac{q(x)}{\int q(x) \, dx}.
\]
The training objective is to minimize the negative log-likelihood loss, defined as:
\[
\mathcal{L} = - \mathbb{E}_{x \sim p_{\mathrm{data}}} \left[\log \hat{q}(x)\right].
\]
In practice, given a set of training samples \(\{x_i\}_{i=1}^N\), this expectation is approximated by the empirical mean:
\[
\mathcal{L} \approx - \frac{1}{N} \sum_{i=1}^N \log \hat{q}(x_i).
\]
Theoretically, for relatively smooth density functions, standard function approximators are often sufficient. However, for more complex density functions with rapid oscillations or sharp variations, it is necessary to employ models with stronger high-frequency function fitting capabilities.

\textit{Datasets.} We adopt two synthetic distributions, $P(X)_1$ and $P(X)_2$, each comprising 25 Gaussian components. The means $\mu$ of the Gaussian components are uniformly sampled from $[-8, 8]$ for both distributions. For $P(X)_1$, the standard deviations $\sigma$ are uniformly sampled from a wider range $[0.08, 1]$, resulting in relatively smoother, low-frequency variations. In contrast, $P(X)_2$ uses a narrower standard deviation range $[0.08, 0.1]$, producing more abrupt, high-frequency variations. A total of 1,024 samples are drawn from each distribution for training and evaluation.

\textit{Baselines.}
We compare Q-RUN with several baseline models that are commonly adopted in IR tasks. These baselines are grouped into two categories. The first group includes Fourier-based models: \textbf{SIREN}, \textbf{RFF}, \textbf{FAN}, and \textbf{FNO}. The second group consists of other architectures: standard MLPs with \textbf{Tanh} and \textbf{ReLU} activations and \textbf{PWLNN}. KANs and QIREN are not reported for DE because KANs perform poorly on this task, and scaling up QIREN for DE is difficult on classical computers.

\begin{figure}[t]
    \centering
    \includegraphics[width=1\linewidth]{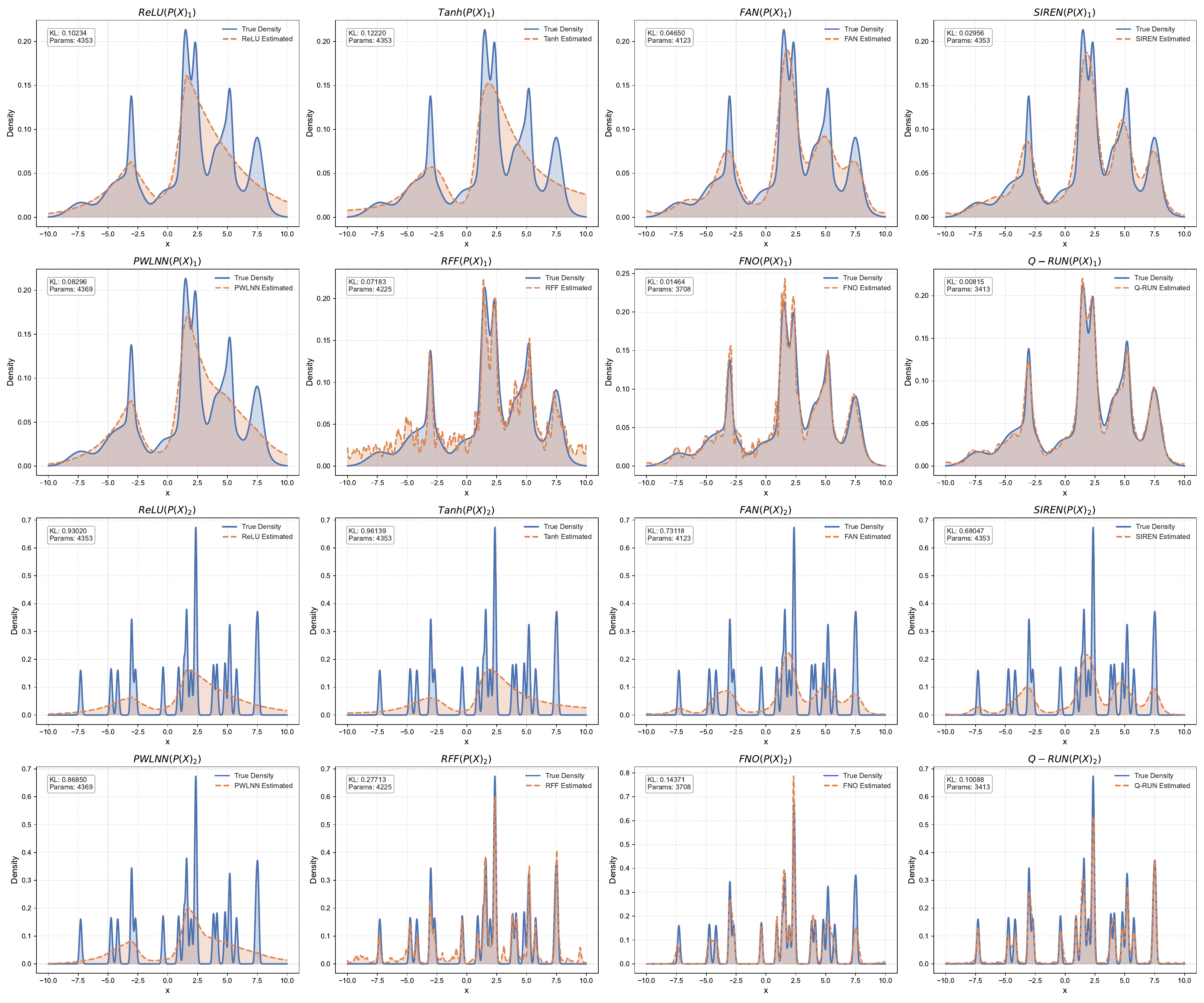}
\caption{Performance of various methods on the Density Estimation task, including the reported KL divergence values and corresponding model parameter counts.}        \label{F_A_3_3}
\end{figure}

\textit{Results.} The visualized experimental results for DE are shown in Figure~\ref{F_A_3_3}.
We observe that standard MLPs with ReLU and Tanh perform poorly on $P(X)_1$ and $P(X)_2$, with PWLNN offering only marginal improvements. Fourier-based methods achieve significantly better performance. Specifically, SIREN and FAN behave similarly, performing better on $P(X)_1$ than $P(X)_2$, likely due to the higher frequency components in $P(X)_2$. RFF and FNO further improve the results, but they also appear to introduce additional noise. In particular, the instability of RFF may stem from the randomly initialized Fourier features, making it less stable than FNO. In contrast, Q-RUN achieves the best performance on both distributions with the fewest parameters, outperforming all other models in fitting regions with extreme variations, while avoiding the noise issues seen in RFF and FNO.

\subsubsection{Energy Modeling}

Energy Modeling (EM) trains a regressor $E_\theta: \mathcal{M} \rightarrow \mathbb{R}$ to predict molecular potential energy from atomic coordinates $\mathcal{M}$. 
This task requires a specialized loss function. Let \( f_\theta({x}) \) denote the predicted potential energy for molecular coordinates \({x}\). The predicted force \(\hat{\mathbf{F}}\) is obtained as the negative gradient of the energy with respect to the coordinates:
\begin{equation*}
    \hat{E} = f_\theta({x}),
\end{equation*}
\begin{equation*}
    \hat{\mathbf{F}} = - \nabla_{{x}} \hat{E} = - \nabla_{{x}} f_\theta({x}).
\end{equation*}
The loss function consists of two parts: the MSE of the predicted energy and the MSE of the predicted forces:
\begin{equation*}
    \mathcal{L}_{\text{energy}} = \frac{1}{N} \sum_{i=1}^N \left( \hat{E}_i - E_i \right)^2,
\end{equation*}
\begin{equation*}
    \mathcal{L}_{\text{force}} = \frac{1}{N} \sum_{i=1}^N \left\| \hat{\mathbf{F}}_i - \mathbf{F}_i \right\|^2,
\end{equation*}
where \(E_i\) and \(\mathbf{F}_i\) denote the ground truth energy and forces, respectively.
The total training loss is a weighted sum of these two terms:
\begin{equation*}
    \mathcal{L} = \mathcal{L}_{\text{energy}} + \beta\times\mathcal{L}_{\text{force}}.
\end{equation*}

\textit{Datasets.} This dataset~\cite{le2025symmetry} contains molecular configurations of water molecules (\( \mathrm{H_2O} \)), including the more complex \( \mathrm{H_2O} \) dimer system, where two water molecules interact. The \( \mathrm{H_2O} \) subset consists of 1,000 single water molecule configurations, each annotated with total energy and atomic forces along all coordinate directions. The \( \mathrm{H_2O} \) dimer subset contains 10,000 dimer configurations with corresponding energy and force labels. The learning task is to regress the potential energy from atomic positions, serving as a standard benchmark for molecular energy prediction models. Following standard practice, 20\% of the dataset is held out for testing. Before training, we preprocess the data by first applying normalization, followed by centering.

\textit{Baselines.}
We compare Q-RUN against several baselines, which are the same as those used for DE, including Fourier-based models (\textbf{SIREN}, \textbf{RFF}, \textbf{FAN}, \textbf{FNO}) and MLPs with \textbf{Tanh} and \textbf{ReLU} activations, as well as \textbf{PWLNN}. Due to the increased complexity of this task, KANs still fail to achieve good performance, and QIREN remains infeasible to run on classical computers because of its required model scale.

\begin{table}[t]
\setlength{\tabcolsep}{1pt}
\centering
\caption{Energy Modeling results measured by Mean Absolute Error (MAE). Lower values indicate better performance.}
\begin{tabular}{c cccccccc}
\toprule
Dataset & Tanh & ReLU & PWLNN & FNO & RFF & SIREN & FAN & Q-RUN \\
\midrule
\#Params 
& 4,061 & 4,061 & 4,689 & 4,028 & 16,705 & 4,061 & 4,551 & \cellcolor{gray!20}3,493 \\
H\textsubscript{2}O 
& $3.1 \times 10^{-3}$ & $4.9 \times 10^{-2}$ & $4.9 \times 10^{-2}$ & $2.0 \times 10^{-3}$ & $5.5 \times 10^{-2}$ & $2.2 \times 10^{-4}$ & $9.5 \times 10^{-5}$ & \cellcolor{gray!20}$6.9 \times 10^{-6}$ \\
H\textsubscript{2}O Dimer 
& $3.2 \times 10^{-4}$ & $2.1 \times 10^{-1}$ & $2.1 \times 10^{-1}$ & $7.5 \times 10^{-4}$ & $2.0 \times 10^{-4}$ & $1.7 \times 10^{-4}$ & $2.9 \times 10^{-4}$ & \cellcolor{gray!20}$3.1 \times 10^{-5}$ \\
\bottomrule
\end{tabular}
\label{tb:em}
\end{table}

\textit{Results.} Experimental results for EM are shown in Table~\ref{tb:em}. Methods based on ReLU perform the worst. Although PWLNN can be considered an enhanced version of ReLU, it degrades to ReLU-level performance on the EM task. In contrast, Tanh performs somewhat better due to its suitability for handling symmetric molecular data. Methods incorporating Fourier features, such as SIREN and FAN, remain competitive. Nevertheless, Q-RUN achieves superior performance. Notably, Q-RUN achieves up to a 449× improvement over Tanh and approximately a 13× gain over the best-performing FAN, all while using the fewest trainable parameters under comparable model scales. This superior performance is attributed to Q-RUN’s quantum-inspired Fourier series modeling capacity, which enables effective approximation of complex, high-frequency energy functions with fewer parameters.

\subsection{RQ3: Experiments on Predictive Modeling Tasks}

In this subsection, we explore representative predictive modeling tasks spanning key AI domains, including natural language processing, computer vision, and time-series analysis, and we additionally consider the parameter-efficient fine-tuning task.

\subsubsection{Time Series Forecasting}

\begin{table*}[t]
\caption{Performance of different sequence models on time series forecasting tasks. Model parameters are indicated below each model name.}
\centering
\renewcommand{\arraystretch}{1.3}
\setlength{\tabcolsep}{3.5pt}
\begin{tabular}{@{}cccccccccccc|cc@{}}
\toprule
\multirow{2}{*}{Dataset} & \multirow{2}{*}{Output} 
& \multicolumn{2}{c}{\shortstack{LSTM\\(12.51M)}} 
& \multicolumn{2}{c}{\shortstack{Mamba\\(12.69M)}} 
& \multicolumn{2}{c}{\shortstack{Transformer\\(12.12M)}} 
& \multicolumn{2}{c}{\shortstack{FANGated\\(11.07M)}} 
& \multicolumn{2}{c}{\shortstack{FAN\\(11.06M)}} 
& \multicolumn{2}{c}{\shortstack{Q-RUN\\(10.48M)}} \\
\cmidrule(lr){3-4} \cmidrule(lr){5-6} \cmidrule(lr){7-8} \cmidrule(lr){9-10} \cmidrule(lr){11-12} \cmidrule(lr){13-14}
& & MSE & MAE & MSE & MAE & MSE & MAE & MSE & MAE & MSE & MAE & MSE & MAE  \\
\midrule
\multirow{4}{*}{Weather} 
 & 96  & 1.069 & 0.742 & 0.552 & 0.519 & 0.413 & 0.438 & 0.292 & 0.380 & 0.313 & 0.431 & \cellcolor{gray!20}0.247 & \cellcolor{gray!20}0.341  \\
 & 192 & 1.090 & 0.778 & 0.700 & 0.595 & 0.582 & 0.540 & 0.535 & 0.550 & 0.472 & 0.525 & \cellcolor{gray!20}0.323 & \cellcolor{gray!20}0.402  \\
 & 336 & 0.992 & 0.727 & 0.841 & 0.667 & 0.751 & 0.626 & 0.637 & 0.602 & 0.719 & 0.581 & \cellcolor{gray!20}0.394 & \cellcolor{gray!20}0.455  \\
 & 720 & 1.391 & 0.892 & 1.171 & 0.803 & 0.967 & 0.715 & 0.845 & 0.706 & 0.732 & 0.670 & \cellcolor{gray!20}0.493 & \cellcolor{gray!20}0.522  \\
\midrule
\multirow{4}{*}{Exchange} 
 & 96  & 0.938 & 0.794 & 0.908 & 0.748 & 0.777 & 0.681 & 0.685 & 0.644 & 0.657 & 0.623 & \cellcolor{gray!20}0.522 & \cellcolor{gray!20}0.588  \\
 & 192 & 1.241 & 0.899 & 1.328 & 0.925 & 1.099 & 0.800 & 0.998 & 0.757 & 0.968 & 0.741 & \cellcolor{gray!20}0.693 & \cellcolor{gray!20}0.688  \\
 & 336 & 1.645 & 1.048 & 1.512 & 0.992 & 1.614 & 1.029 & 1.511 & 0.961 & 1.266 & 0.905 & \cellcolor{gray!20}0.936 & \cellcolor{gray!20}0.799 \\
 & 720 & 1.949 & 1.170 & 2.350 & 1.271 & 2.163 & 1.204 & 1.658 & 1.104 & 1.857 & 1.145 & \cellcolor{gray!20}1.284 & \cellcolor{gray!20}0.922  \\
\midrule
\multirow{4}{*}{Traffic} 
 & 96  & 0.659 & 0.359 & 0.666 & 0.377 & 0.656 & 0.357 & 0.647 & 0.355 & 0.643 & \cellcolor{gray!20}0.347 & \cellcolor{gray!20}0.635 & 0.353  \\
 & 192 & 0.668 & 0.360 & 0.671 & 0.381 & 0.672 & 0.363 & 0.649 & 0.353 & 0.657 & 0.354 & \cellcolor{gray!20}0.642 & \cellcolor{gray!20}0.345  \\
 & 336 & \cellcolor{gray!20}0.644 & \cellcolor{gray!20}0.342 & 0.665 & 0.374 & 0.673 & 0.360 & 0.665 & 0.358 & 0.656 & 0.353 & 0.661 & 0.353 \\
 & 720 & \cellcolor{gray!20}0.654 & \cellcolor{gray!20}0.351 & 0.662 & 0.364 & 0.701 & 0.380 & 0.682 & 0.369 & 0.673 & 0.363 & 0.669 & 0.356  \\
\midrule
\multirow{4}{*}{ETTh} 
 & 96  & 0.999 & 0.738 & 0.860 & 0.697 & 1.139 & 0.853 & 0.842 & 0.736 & 0.873 & 0.707 & \cellcolor{gray!20}0.786 & \cellcolor{gray!20}0.677  \\
 & 192 & 1.059 & 0.759 & 0.849 & \cellcolor{gray!20}0.700 & 1.373 & 0.932 & 0.885 & 0.748 & 0.914 & 0.741 & \cellcolor{gray!20}0.833 & 0.707  \\
 & 336 & 1.147 & 0.820 & 1.005 & 0.745 & 1.261 & 0.924 & 0.980 & \cellcolor{gray!20}0.770 & 0.999 & 0.793 & \cellcolor{gray!20}0.947 & \cellcolor{gray!20}0.770  \\
 & 720 & 1.206 & 0.847 & 0.994 & \cellcolor{gray!20}0.758 & 1.056 & 0.819 & 1.002 & 0.798 & 1.031 & 0.818 & \cellcolor{gray!20}0.932 & 0.776 \\
\bottomrule
\end{tabular}
\label{t_A_3_2_6}
\end{table*}

Time Series Forecasting (TSF) predicts the next [96, 192, 336, 720] steps from an input sequence of 96 time steps.

\textit{Datasets.} We use common benchmark datasets, including:

\begin{itemize}
    \item \textbf{Weather:} Contains 21 meteorological indicators such as air temperature and humidity, recorded every 10 minutes throughout 2020.
    \item \textbf{Exchange:} Contains daily foreign exchange rates from eight countries collected between 1990 and 2016.
    \item \textbf{Traffic:} Describes road occupancy rates with hourly data recorded by San Francisco highway sensors from 2015 to 2016.
    \item \textbf{ETTh:} Consists of two hourly datasets covering load features of seven kinds of petroleum and power transformers from July 2016 to July 2018.
\end{itemize}

\textit{Baselines.}
We compare Q-RUN against well-established architectures:
\begin{itemize}
    \item \textbf{LSTM}: A classic recurrent architecture for time series and sequence modeling.
    \item \textbf{Mamba} \cite{gu2023mamba}: A structured state-space model with efficient sequence modeling capabilities.
    \item \textbf{Transformer} \cite{vaswani2017attention}: A general-purpose attention-based architecture. 
    \item \textbf{FAN} and \textbf{FANGated} \cite{dong2024fan}: Both designed as plug-and-play state-of-the-art modules applicable across a wide range of AI tasks. FANGated uses a gating mechanism to control information flow between standard FCs and FAN components. They are implemented as replacements for the feed-forward networks (FFN) in Transformer, and Q-RUN is integrated using the same strategy.
\end{itemize}
Notably, our goal is not to outperform task-specific, highly specialized models, but to evaluate Q-RUN as a general-purpose, plug-and-play module for enhancing mainstream AI architectures. Accordingly, we compare it with representative baselines and modular methods such as FAN and FANGated, rather than domain-specific state-of-the-art solutions.

\textit{Results.}
Experimental results for the TSF tasks are summarized in Table~\ref{t_A_3_2_6}.
Specifically, traditional models like LSTM, Mamba, and Transformer show mixed performance across different datasets. Replacing FFN with the state-of-the-art model FAN or FANGated improves performance and reduces learnable parameters. Compared to FAN, FANGated’s hybrid MLP-Fourier design underperforms, indicating TSF benefits more from pure Fourier modeling. Our Q-RUN also leverages Fourier-based modeling, but draws inspiration from DRQC, providing enhanced high-frequency fitting capabilities. As a result, Q-RUN consistently achieves the best performance across most datasets while using the fewest parameters.

\begin{table}[t]
\caption{Comparison of model performance on language modeling sentiment classification tasks. Accuracy is reported on SST-2, IMDB, Sentiment140, and Amazon Reviews datasets. Higher values indicate better performance.}
\centering
\renewcommand{\arraystretch}{1.1}
\setlength{\tabcolsep}{6pt}
\begin{tabular}{cccccc}
\toprule
Model & \#Params (M) & SST-2 & IMDB & Sentiment140 & Amazon Reviews \\
\midrule
LSTM        & 120.14 & 80.60 & 64.38 & 59.79 & 71.52 \\
Mamba       & 129.73 & 79.59 & 62.03 & 58.74 & 67.19 \\
BERT & 109.48 & 81.19 & 69.94 & 57.79 & 71.55 \\
FANGated    & 95.33  & 80.39 & 70.12 & \cellcolor{gray!20}61.94 & 76.89 \\
FAN         & 95.32  & 81.54 & \cellcolor{gray!20}73.98 & 60.93 & 77.63 \\
\hline
Q-RUN       & \cellcolor{gray!20}66.99 & \cellcolor{gray!20}82.11 & 73.12 & 61.37 & \cellcolor{gray!20}78.21 \\
\bottomrule
\end{tabular}
\label{tb2_lm}
\end{table}

 %Q-RUN在数据建模类任务取得优异表现的同时，我们自然更加好奇Q-RUN这种能力是否能在更广泛的AI任务中展现优势。

%Setup: We 同样consider three fundamental data 决策tasks, which span diverse目前最受关注NLP、CV以及信号等领域. Formally, these tasks are defined as follows:

%时间序列预测：..... 数据集包括Weather Exchange Traffic ETTh

%语言建模： 具体来说是情感分类 .....                 其中仅使用SST-2训练模型，在 SST-2 IMDB Sent140 Amazon等数据集上测试。

%图像识别：.......数据集包括 MNIST MNIST-M F-MNIST F-MNIST-C

%总体的实验设置我们遵循FAN，关于更For more detailed descriptions of the tasks and datasets, please refer to Section S3 in the supplementary material. 此外S4中我们还考虑了LORA任务。

% 基线模型方面，我们首先考虑了标准模型，如 LSTM、Mamba 以及 Transformer（在语言建模任务中指的是基于 Transformer 的 BERT）。其次，我们纳入了 FAN 和 FANGated 方法进行比较，其中 FANGated 通过门控机制在普通全连接层与 FAN 之间动态调控信息流。这两种方法均作为 Transformer（或 BERT）中部分全连接层的替代结构进行实验，我们提出的 Q-RUN 同样以该方式进行替换。需要特别指出的是，我们的目标在于评估 Q-RUN 作为通用即插即用模块对主流 AI 模型的增益能力。因此，我们所选的比较对象是具有代表性的标准结构和可插拔方法，如 FAN 与 FANGated，而不是各个具体任务中经过高度定制优化的最优专业化模型

\subsubsection{Language Modeling}

Language Modeling (LM) in this study focuses on sentiment classification, where the goal is to predict the sentiment label of a given text. This task requires capturing both syntactic and semantic information from the input sequence. 

\textit{Datasets.} We select several widely used benchmark datasets in the sentiment classification domain. Models are trained solely on SST-2 and evaluated not only on the SST-2 test set but also on IMDB, Sentiment140, and Amazon Reviews to assess their zero-shot generalization. These datasets are summarized as follows.
\begin{itemize}
  \item \textbf{SST-2} (Stanford Sentiment Treebank 2): A widely used sentiment analysis dataset consisting of 11,855 sentences from movie reviews, with binary labels (positive or negative) at sentence level. 

  \item \textbf{IMDB:} Large-scale movie review dataset for binary sentiment classification, containing approximately 50,000 labeled reviews split evenly into positive and negative categories. It is widely used in NLP research.

  \item \textbf{Sentiment140:} A Twitter sentiment dataset with 1.6 million tweets automatically labeled as positive or negative based on emoticons. Contains roughly 800,000 positive and 800,000 negative tweets. 

  \item \textbf{Amazon Reviews:} A large-scale dataset comprising millions of reviews, including ratings, review texts, timestamps, and product metadata. Widely used for sentiment classification and recommendation research. 
\end{itemize}

\textit{Baselines.} Similar to the TSF task, we compare Q-RUN against well-established architectures, including \textbf{LSTM}, \textbf{Mamba}, \textbf{FAN}, and \textbf{FANGated}. We also include \textbf{BERT}~\cite{devlin2019bert}, a widely used Transformer-based language model. For FAN, FANGated, and Q-RUN, we replace the feed-forward networks (FFN) in BERT with the corresponding modules to assess their effectiveness as drop-in components.

\textit{Results.}
Experimental results are shown in Table~\ref{tb2_lm}. On the LM task, Mamba performs slightly worse than LSTM, while Transformer generally yields better results. Replacing the FFN with FAN or FANGated yields further gains, while substituting with Q-RUN achieves competitive performance in most cases. Remarkably, Q-RUN attains these improvements while reducing the number of trainable parameters by nearly 40\% compared to the original BERT.

\subsubsection{Parameter-efficient Fine-tuning}
% ---------------------------

\begin{table}[t]
\centering
\caption{Comparison of models in terms of parameter ratios, per-epoch training time, validation accuracy, and test accuracy across different methods. Higher values indicate better performance. For DistilBERT-base, the baseline corresponds to full fine-tuning, while for Mistral-7b, the baseline corresponds to zero-shot inference.}
\setlength{\tabcolsep}{4pt}  % 调整列间距
\renewcommand{\arraystretch}{1.3}  % 调整行高
\resizebox{0.95\linewidth}{!}{%
\begin{tabular}{c|cccc|ccc}
\toprule
\multirow{2}{*}{Method} & \multicolumn{4}{c|}{DistilBERT-base} & \multicolumn{3}{c}{Mistral-7b} \\
                        & \#Params (\%) & Time (min) & Val Acc (\%) & Test Acc (\%) & \#Params (\%) & Time (min) & Test Acc (\%) \\
\midrule
Baseline     & 100.0  & 8.92 & 92.08 & 91.53 & 100.0   & - & 62.12 \\
\hline
LoRA (r=4)           & 0.1105   & 7.25 & 90.70 & 89.47 & 0.0470    & 3.59 & 81.76 \\
LoRA (r=8)           & 0.2195   & 7.36 & 90.08 & 89.09 & 0.0939    & 3.61 & 82.02 \\
\hline
DoRA (r=4)           & 0.1240   & 7.32 & 90.78 & 90.60 & 0.0515    & 3.69 & 80.61\\
DoRA (r=8)           & 0.2330   & 9.49 & 90.66 & 90.28 & 0.0985    & 3.88 & 83.16 \\
\hline
Q-RUN (r=4)          & 0.1165   & 7.51 & 91.12 & 90.40 & 0.0476    & 3.73 & 82.65 \\
Q-RUN (r=8)          & 0.2255   & 7.66 & \cellcolor{gray!20}\textbf{91.28} & \cellcolor{gray!20}\textbf{90.73} & 0.0945    & 3.74 & \cellcolor{gray!20}\textbf{84.18} \\
\bottomrule
\end{tabular}%
}
\label{tab:lora}
\end{table}

We explore the potential of Q-RUN to enhance parameter-efficient fine-tuning in this subsection. This motivation is inspired by recent advances in quantum fine-tuning of pre-trained language models, where variational quantum circuits have demonstrated strong capabilities in adapting large models with fewer trainable parameters \cite{liu2024quantum,kong2025quantum,chen2024quanta}. As a classically implementable method inspired by quantum machine learning, Q-RUN is theoretically expected to exhibit similar fine-tuning capabilities. Moreover, since it does not rely on quantum hardware, Q-RUN may offer improved efficiency and faster deployment in practical large language model adaptation scenarios. Compared to purely quantum approaches, Q-RUN avoids the limitations of current quantum devices, such as noise and scalability constraints, while retaining key advantages derived from QML.

\paragraph{Setup}

We conduct parameter-efficient fine-tuning experiments on two representative language models.
\begin{itemize}
    \item \textbf{Fine-tuning DistilBERT on IMDB:} 
    We fine-tune a pre-trained DistilBERT model \cite{sanh2019distilbert} on the IMDB sentiment classification dataset. During fine-tuning, only the adapter module is trained, while the base DistilBERT weights remain frozen.
    
    \item \textbf{Fine-tuning Mistral-7B on Irony:} 
    Mistral-7B \cite{jiang2023diego} is a 7-billion-parameter open-source language model featuring Grouped-Query Attention and Sliding Window Attention. We fine-tune it on the irony subset of TweetEval \cite{van2018semeval}, which contains 2.86k training samples and 784 test samples. This small dataset allows for rapid validation.
\end{itemize}

\paragraph{Baselines}  
We compare Q-RUN with two representative parameter-efficient fine-tuning methods: LoRA \cite{hu2022lora} and DoRA \cite{liu2024dora}.  These methods have demonstrated strong empirical performance in fine-tuning large language models efficiently, making them strong baselines for our evaluation.

\begin{itemize}
    \item \textbf{LoRA} (Low-Rank Adaptation) introduces low-rank decomposition matrices to the weights of pre-trained models, reducing the number of trainable parameters by injecting trainable low-rank matrices while freezing the original weights.
    \item \textbf{DoRA} (Weight-Decomposed Low-Rank Adaptation) improves LoRA by decoupling pre-trained weights into magnitude and direction components. 

\end{itemize}

For integrating Q-RUN into the LoRA layer, the $f_0$ layer in the original Q-RUN design can be omitted, since the matrix $A$ in LoRA already reduces the input dimensionality to $r$. The resulting $r$-dimensional vector is then passed through the Q-RUN layer, which outputs an $r$-dimensional feature. This feature is subsequently projected back to the original hidden dimension through the matrix $B$. For large language models, however, directly replacing the LoRA structure with Q-RUN may introduce excessive structural modifications and potentially disrupt the rich prior knowledge embedded in the model. To address this, we adopt a modulation strategy in which the outputs of the original LoRA branch and the Q-RUN branch are combined, similar to the approach of \cite{kong2025quantum}.

\paragraph{Results}  
Table~\ref{tab:lora} summarizes the results of Q-RUN and various parameter-efficient fine-tuning methods. Across all settings, Q-RUN consistently outperforms other baselines in terms of test accuracy while maintaining competitive training time and parameter efficiency.
On DistilBERT, Q-RUN with rank $r=8$ achieves the best test accuracy of 90.73\%, surpassing both LoRA and DoRA at the same rank. Even with $r=4$, Q-RUN reaches 90.40\%, already competitive with DoRA ($r=8$) and significantly better than LoRA ($r=4,r=8$). Notably, Q-RUN requires only 0.151M trainable parameters, which is less than 0.23\% of full fine-tuning (66.96M), demonstrating strong parameter efficiency.
On Mistral-7b, Q-RUN achieves the highest test accuracy of 84.18\% ($r=8$), outperforming all other methods, whereas direct inference with Mistral-7b alone yields only 62.12\%. This demonstrates that the inductive bias introduced by Q-RUN has a positive impact even when applied to a model with 7B parameters. In terms of training time, Q-RUN remains comparable to LoRA and DoRA, with only a marginal overhead introduced by the additional element-wise MLP, and in some cases, it even trains faster than DoRA.
Overall, these results confirm that Q-RUN achieves a superior trade-off between performance and parameter efficiency compared to LoRA and DoRA.

\subsubsection{Image Classification}

Finally, we further evaluate Q-RUN on the Image Classification (IC) task, a fundamental benchmark in computer vision. This task requires assigning each input image to one of several predefined categories.

\begin{table}[t]
\centering
\renewcommand{\arraystretch}{1.1}
\caption{Comparison of model performance and parameter counts on image classification benchmarks. The reported metric is accuracy. Higher values indicate better performance.}
\begin{tabular}{@{}ccccccc@{}}
\toprule
Model & \#Params & MNIST & MNIST-Modified & Fashion-MNIST & Fashion-MNIST-Corrupted \\
\midrule
FAN & 1.86M & 99.67 & 94.23 & 94.47 & 88.82 \\
CNN & 1.86M & 99.63 & \cellcolor{gray!20}94.52 & 94.15 & 88.61 \\
\hline
Q-RUN & \textbf{1.36M} & \cellcolor{gray!20}99.68 & 94.34 & \cellcolor{gray!20}94.57 & \cellcolor{gray!20}88.92 \\
\bottomrule
\end{tabular}
\label{tb_ic}
\end{table}

\textit{Datasets.}
The datasets used include standard benchmarks for image classification and their commonly adopted variants, detailed as follows. 
\begin{itemize}
    \item \textbf{MNIST:} A classic benchmark dataset of handwritten digits (0–9), consisting of 60,000 training and 10,000 test grayscale images of size $28 \times 28$.
    
    \item \textbf{MNIST-Modified:} A modified version of MNIST where digits are blended with color patches from natural images, introducing background textures and RGB information to the original grayscale digits.
    
    \item \textbf{Fashion-MNIST:} A drop-in replacement for MNIST, containing grayscale images of clothing items (e.g., shirts, shoes, bags) with the same size and split as MNIST.
    
    \item \textbf{Fashion-MNIST-Corrupted:} A corrupted version of Fashion-MNIST, including various synthetic corruptions (e.g., noise, blur, pixelation, brightness changes), designed to evaluate model robustness under distribution shifts.
    \end{itemize}

\textit{Baselines.}
Our main baselines are CNNs and FAN. CNNs represent the standard architecture for image classification, using convolutional filters to extract local spatial features. For FAN and Q-RUN, we replace the fully connected layers in CNNs with the respective modules to assess their effectiveness as plug-and-play components for image recognition.

\textit{Results.}
Experimental results for the image classification task are presented in Table~\ref{tb_ic}. It is evident that both FAN and Q-RUN improve the performance of CNNs on MNIST, Fashion-MNIST, and Fashion-MNIST-Corrupted. This demonstrates that Fourier features are effective at capturing high-frequency edges and fine-grained details in images, enabling the model to improve classification accuracy. Notably, Q-RUN further enhances CNN performance compared to FAN while using fewer parameters than the original CNN. An interesting observation is that on MNIST-Modified, both FAN and Q-RUN show a decrease in performance. This is likely because MNIST-Modified overlays natural image color patches on the original digits, introducing substantial irrelevant high-frequency information. As a result, Fourier-based models may inadvertently fit irrelevant high-frequency noise, which leads to reduced accuracy. Nevertheless, Q-RUN still outperforms FAN, indicating its superior ability to extract meaningful features and mitigate the impact of irrelevant high-frequency noise.

%对于TSF任务，传统的LSTM、Mamba和Transformer在不同数据集上的表现互有胜负。而在将其中的FFN替换为FAN明显得到了提升，尤其还节省了可学习参数。相比FAN，尽管FANGated结构更复杂，但效果反而有所下降，这充分说明傅里叶序列拟合能力的引入对TSF是有意义的。Q-RUN与FAN类似，但引入的是DRQC启发的傅里叶序列拟合能力，这种能力更强、更适合建模高频信号，因此在时序预测中的全部数据集上取得了最好的表现，同时使用的参数也是最少的。对于LM任务,Mamba表现的略差，不如LSTM。而Transformer在大多数情况下表现更好，在替换FNN为FAN或FANGated后，明显性能得到了进一步的提升。而在替换为Q-RUN，它的性能在多数情况下都达到了几种方法中的最优。尤其是，替换为Q-RUN后，相比原始Transformer减少了将近一半的可学习参数。在IC中也观察到了类似的结果，Q-RUN 在大多数情况下都能够提升原始 CNN 的表现，同时使用了更少的参数。与当前最先进的 FAN 方法相比，Q-RUN 的改进更加显著。这些实验结果一方面说明了引入傅里叶序列特性对于改进人工智能模型的重要意义，另一方面也展示了 Q-RUN 作为量子启发方法的强大潜力。

\begin{figure}
    \centering
    \includegraphics[width=0.8\linewidth]{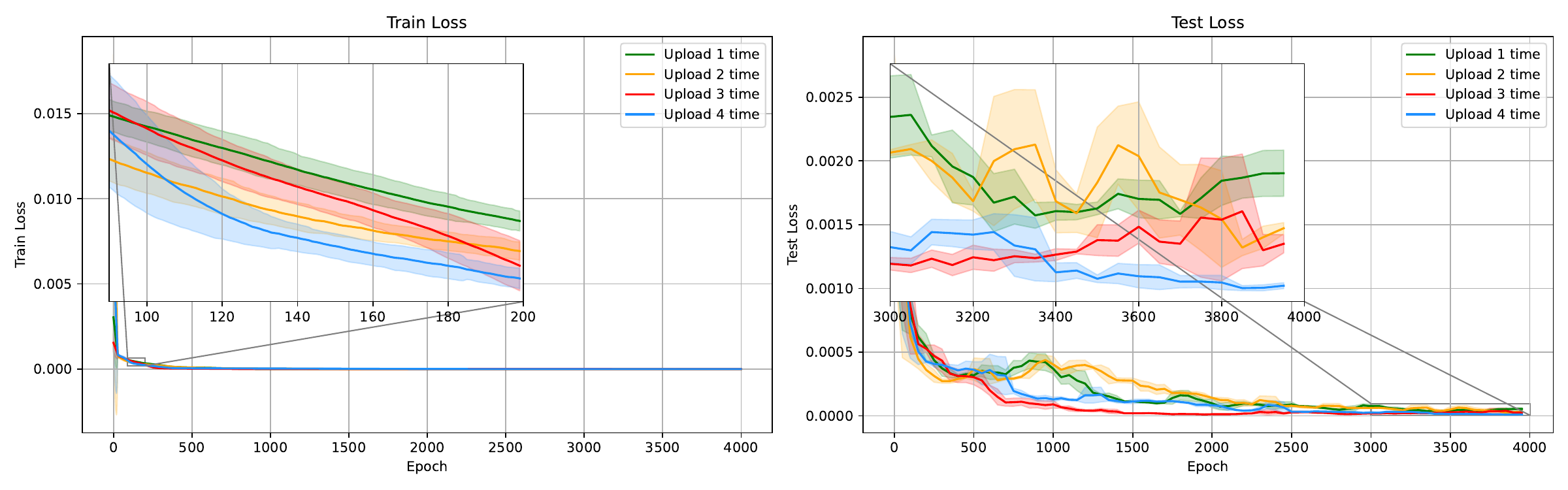}
\caption{Q-RUN’s approximation of DRQC with varying data re-uploading times. The left panel shows performance on the training set, and the right panel shows performance on the test set.}
    \label{ef2}
\end{figure}

\subsection{RQ4: Ablation Studies}

This subsection analyzes the impact of data re-uploads on the proposed Q-RUN model. It also reports Q-RUN’s computational efficiency and examines its scaling behavior.

\subsubsection{Impact of Data Re-Upload}

We analyze the impact of the number of data re-uploads on the modeling capability of Q-RUN. As shown in Figure~\ref{ef2}, building on the experimental setup from RQ1, we tracked the train and test loss across training epochs for Q-RUN with different numbers of data re-uploads. 
On the training set, we observe that in the region where the loss decreases rapidly, circuits with more data re-uploading times exhibit faster loss reduction. On the test set, we focus on the last 1,000 iterations. 
It is evident that increasing the number of re-uploading times from 1 to 2 results in relatively slow convergence, whereas 3 to 4 re-uploading times lead to significantly faster convergence and lower final loss values. These observations indicate that a higher number of data re-uploading times enhances the expressive power of the model.

% Figure~\ref{ef2}(b) further illustrates the performance of Q-RUN at different parameter scales, compared with several representative models. As the number of parameters increases, Q-RUN with depths 3 and 4 consistently improves its performance, surpassing both FAN and MLP by nearly two orders of magnitude on the symbolic regression task. These results demonstrate Q-RUN's strong function approximation ability and excellent scalability. See Section A5 of the Technical Appendix for further ablation studies, including training and inference time.

\begin{table}[t]
\centering
\caption{Training and inference time comparison across models. All experiments were conducted using a PyTorch implementation on an NVIDIA RTX 3060 GPU.}
\setlength{\tabcolsep}{4pt}
\small
\begin{tabular}{cccc}
\toprule
\textbf{Model} & \textbf{\# Params} & \textbf{Train Time (s)} & \textbf{Test Time (s)} \\
\midrule
RFF & 4225  & 6.46     & 0.0010   \\
Tanh    & 4353  & 7.08     & 0.0010   \\
ReLU    & 4353  & 7.18     & 0.0010   \\
SIREN          & 4353  & 8.28    & 0.0011   \\
FAN          & 4123  & 12.12    & 0.0015   \\
PWLNN        & 4369  & 23.25    & 0.0010   \\
FNO          & 3708  & 81.69    & 0.0131   \\
KAN          & 4352  & 3245.51  & 0.0377   \\
QIREN        & \multicolumn{3}{c}{\textit{Out of Memory}}         \\
\midrule
Q-RUN       & 3421  & 32.26    & 0.0020   \\
\bottomrule
\end{tabular}
\label{tab:model_timing}
\end{table}

\subsubsection{Efficiency Comparison}

We conducted systematic evaluations on a density estimation task using a dataset of 10,000 training samples and 256 testing samples. Except for KAN, all models were implemented with 3 layers and hidden widths ranging from 32 to 72, adjusted to ensure comparable total parameter counts. The Adam optimizer was used across all models, while KAN was trained using the LBFGS optimizer due to its design characteristics. For QIREN, memory overflow occurred when attempting to scale its hidden layers to a comparable size. Q-RUN employed 4 data re-uploading times with a layer width of 32, and the Element-wise Observable Module was configured with size [8, 32, 2]. All models were trained for 3,000 epochs under identical conditions.

Table~\ref{tab:model_timing} compares the training and inference time of various models under the same density estimation task and parameter scale. While Q-RUN is slower than standard MLP-based baselines due to its structured design, it is still significantly more efficient than recent specialized architectures, such as FNO and KAN. Notably, the state-of-the-art quantum model QIREN fails to process this task due to excessive memory consumption, indicating poor scalability. In contrast, our quantum-inspired Q-RUN runs stably and efficiently, highlighting its practical value for large-scale problems with reasonable computational overhead.

\begin{figure*}[t]
    \centering
    \includegraphics[width=1\linewidth]{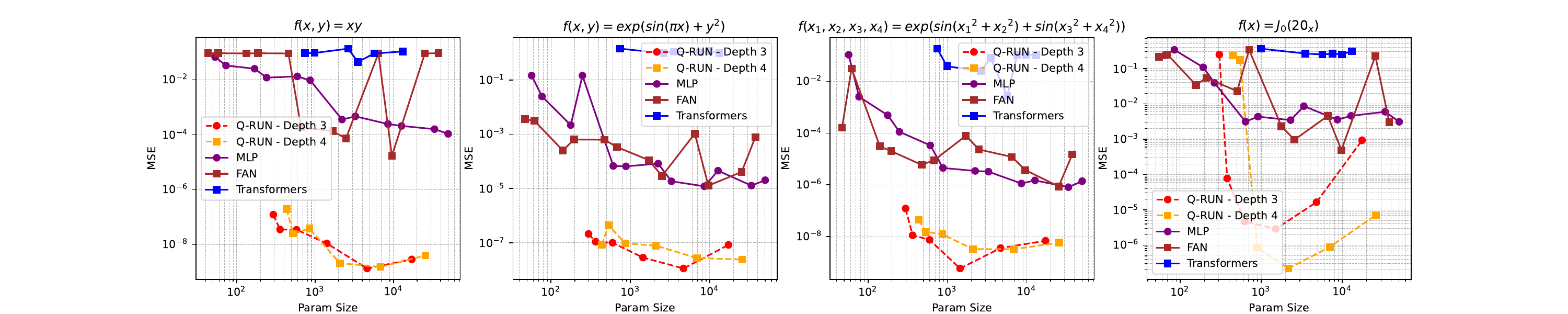}
\caption{The scalability comparison across models.}    \label{F_A5_2}
\end{figure*}

\subsubsection{Scalability Comparison}

We follow the experimental setup from FAN to evaluate the scalability of Q-RUN compared to several baseline models on four symbolic regression tasks. Each dataset contains 4,000 samples, with 1,000 used for testing. For MLP, FAN, and Q-RUN, we evaluate scalability with network depths in \([3, 4]\) and hidden sizes in \([4, 8, 16, 32, 64, 128]\). For the Transformer baseline, we vary the embedding dimension in \([4, 8, 12, 16]\).
All models are individually tuned based on their characteristics and optimized using the LBFGS optimizer. To accelerate training, Q-RUN is first pretrained with AdamW and then fine-tuned using LBFGS. 

The experimental results are shown in Figure~\ref{F_A5_2}.
On the first three datasets, Q-RUN achieves remarkably low MSE from the very beginning, outperforming FAN by 2 to 3 orders of magnitude. As the model size increases, its performance continues to improve. On the fourth dataset, however, Q-RUN exhibits a less ideal U-shaped trend. The MSE first decreases and then increases, indicating potential overfitting at larger capacities. Nevertheless, Q-RUN still surpasses other methods with similar parameter scales.

\section{Conclusion and Future Work}

In this paper, we show that QML can inspire the design of better AI models, even without access to quantum hardware. We embed the inductive bias of quantum models into classical neural networks, enabling Q-RUN to be theoretically viewed as a truncated Fourier series with enhanced Fourier expressiveness over standard methods. By relaxing certain quantum constraints, Q-RUN can serve as a practical alternative to standard fully connected layers. Across a wide range of AI tasks, Q-RUN achieves superior performance while using fewer learnable parameters. Our work offers a novel perspective on how to leverage potential quantum advantages in the absence of quantum hardware.

Future work can proceed along several directions. 
First, while we have demonstrated the ability of Q-RUN to approximate DRQC, an open question is whether Q-RUN can be extended to more general quantum circuit architectures to learn circuit behaviors and even emulate quantum outputs. This direction holds promise for establishing Q-RUN as a versatile surrogate for simulating a wider class of quantum models. 
Second, applying Q-RUN to parameter-efficient fine-tuning of large language models represents another promising avenue. By providing inductive bias comparable to that of quantum neural networks, Q-RUN may achieve competitive performance while remaining more practical to deploy on classical hardware. 
Third, it is worth exploring whether Q-RUN can serve as a classical twin of DRQC. In this case, efficiently training Q-RUN on classical devices could facilitate the extension of DRQC on real quantum hardware, thereby offering a cost-effective path toward advancing QML in practice.

%Bibliography
\bibliographystyle{unsrt}  
\bibliography{references}  

@article{yu2022power,
  title={Power and limitations of single-qubit native quantum neural networks},
  author={Yu, Zhan and Yao, Hongshun and Li, Mujin and Wang, Xin},
  journal={Advances in Neural Information Processing Systems},
  volume={35},
  pages={27810--27823},
  year={2022}
}

@article{schuld2021effect,
  title={Effect of data encoding on the expressive power of variational quantum-machine-learning models},
  author={Schuld, Maria and Sweke, Ryan and Meyer, Johannes Jakob},
  journal={Physical Review A},
  volume={103},
  number={3},
  pages={032430},
  year={2021},
  publisher={APS}
}

@inproceedings{
zhao2024quantum,
title={Quantum Implicit Neural Representations},
author={Jiaming Zhao and Wenbo Qiao and Peng Zhang and Hui Gao},
booktitle={Forty-first International Conference on Machine Learning},
year={2024},
url={https://openreview.net/forum?id=50vc4HBuKU}
}

@article{biamonte2017quantum,
  title={Quantum machine learning},
  author={Biamonte, Jacob and Wittek, Peter and Pancotti, Nicola and Rebentrost, Patrick and Wiebe, Nathan and Lloyd, Seth},
  journal={Nature},
  volume={549},
  number={7671},
  pages={195--202},
  year={2017},
  publisher={Nature Publishing Group UK London}
}

@article{perez2020data,
  title={Data re-uploading for a universal quantum classifier},
  author={P{\'e}rez-Salinas, Adri{\'a}n and Cervera-Lierta, Alba and Gil-Fuster, Elies and Latorre, Jos{\'e} I},
  journal={Quantum},
  volume={4},
  pages={226},
  year={2020},
  publisher={Verein zur F{\"o}rderung des Open Access Publizierens in den Quantenwissenschaften}
}

@article{xu2019frequency,
  title={Frequency principle: Fourier analysis sheds light on deep neural networks},
  author={Xu, Zhi-Qin John and Zhang, Yaoyu and Luo, Tao and Xiao, Yanyang and Ma, Zheng},
  journal={arXiv preprint arXiv:1901.06523},
  year={2019}
}

@article{sitzmann2020implicit,
  title={Implicit neural representations with periodic activation functions},
  author={Sitzmann, Vincent and Martel, Julien and Bergman, Alexander and Lindell, David and Wetzstein, Gordon},
  journal={Advances in neural information processing systems},
  volume={33},
  pages={7462--7473},
  year={2020}
}

@article{mitarai2018quantum,
  title={Quantum circuit learning},
  author={Mitarai, Kosuke and Negoro, Makoto and Kitagawa, Masahiro and Fujii, Keisuke},
  journal={Physical Review A},
  volume={98},
  number={3},
  pages={032309},
  year={2018},
  publisher={APS}
}

@article{tancik2020fourier,
  title={Fourier features let networks learn high frequency functions in low dimensional domains},
  author={Tancik, Matthew and Srinivasan, Pratul and Mildenhall, Ben and Fridovich-Keil, Sara and Raghavan, Nithin and Singhal, Utkarsh and Ramamoorthi, Ravi and Barron, Jonathan and Ng, Ren},
  journal={Advances in neural information processing systems},
  volume={33},
  pages={7537--7547},
  year={2020}
}

@article{kovachki2021neural,
   author    = {Nikola B. Kovachki and
                  Zongyi Li and
                  Burigede Liu and
                  Kamyar Azizzadenesheli and
                  Kaushik Bhattacharya and
                  Andrew M. Stuart and
                  Anima Anandkumar},
   title     = {Neural Operator: Learning Maps Between Function Spaces},
   journal   = {CoRR},
   volume    = {abs/2108.08481},
   year      = {2021},
}

@article{tao2022piecewise,
  title={Piecewise linear neural networks and deep learning},
  author={Tao, Qinghua and Li, Li and Huang, Xiaolin and Xi, Xiangming and Wang, Shuning and Suykens, Johan AK},
  journal={Nature Reviews Methods Primers},
  volume={2},
  number={1},
  pages={42},
  year={2022},
  publisher={Nature Publishing Group UK London}
}

@inproceedings{qiao2025quantum,
  title={Quantum Time-index Models with Reservoir for Time Series Forecasting},
  author={Qiao, Wenbo and Zhao, Jiaming and Zhang, Peng},
  booktitle={Proceedings of the 31st ACM SIGKDD Conference on Knowledge Discovery and Data Mining V. 1},
  pages={1173--1184},
  year={2025}
}

@article{lecun2015deep,
  title={Deep learning},
  author={LeCun, Yann and Bengio, Yoshua and Hinton, Geoffrey},
  journal={nature},
  volume={521},
  number={7553},
  pages={436--444},
  year={2015},
  publisher={Nature Publishing Group UK London}
}

@inproceedings{rahaman2019spectral,
  title={On the spectral bias of neural networks},
  author={Rahaman, Nasim and Baratin, Aristide and Arpit, Devansh and Draxler, Felix and Lin, Min and Hamprecht, Fred and Bengio, Yoshua and Courville, Aaron},
  booktitle={International conference on machine learning},
  pages={5301--5310},
  year={2019},
  organization={PMLR}
}

@article{thompson2020computational,
  title={The computational limits of deep learning},
  author={Thompson, Neil C and Greenewald, Kristjan and Lee, Keeheon and Manso, Gabriel F and others},
  journal={arXiv preprint arXiv:2007.05558},
  volume={10},
  year={2020}
}

@inproceedings{woo2023learning,
  title={Learning deep time-index models for time series forecasting},
  author={Woo, Gerald and Liu, Chenghao and Sahoo, Doyen and Kumar, Akshat and Hoi, Steven},
  booktitle={International Conference on Machine Learning},
  pages={37217--37237},
  year={2023},
  organization={PMLR}
}

@article{hua2024fourier,
  title={Fourier Position Embedding: Enhancing Attention's Periodic Extension for Length Generalization},
  author={Hua, Ermo and Jiang, Che and Lv, Xingtai and Zhang, Kaiyan and Ding, Ning and Sun, Youbang and Qi, Biqing and Fan, Yuchen and Zhu, Xuekai and Zhou, Bowen},
  journal={arXiv preprint arXiv:2412.17739},
  year={2024}
}

@article{dong2024fan,
  title={FAN: Fourier Analysis Networks},
  author={Yihong Dong and Ge Li and Yongding Tao and Xue Jiang and Kechi Zhang and Jia Li and Jing Su and Jun Zhang and Jingjing Xu},
  journal={arXiv preprint arXiv:2410.02675},
  year={2024}
}

@article{liu2024kan,
  title={Kan: Kolmogorov-arnold networks},
  author={Liu, Ziming and Wang, Yixuan and Vaidya, Sachin and Ruehle, Fabian and Halverson, James and Solja{\v{c}}i{\'c}, Marin and Hou, Thomas Y and Tegmark, Max},
  journal={arXiv preprint arXiv:2404.19756},
  year={2024}
}

@inproceedings{qiao2024quantum,
  title={Quantum Topic Model: Topic Modeling Using Variational Quantum Circuits},
  author={Qiao, Wenbo and Zhang, Peng and Zhao, Jiaming and Yang, Chang},
  booktitle={ICASSP 2024-2024 IEEE International Conference on Acoustics, Speech and Signal Processing (ICASSP)},
  pages={5895--5899},
  year={2024},
  organization={IEEE}
}

@article{coecke2020foundations,
  title={Foundations for near-term quantum natural language processing},
  author={Coecke, Bob and de Felice, Giovanni and Meichanetzidis, Konstantinos and Toumi, Alexis},
  journal={arXiv preprint arXiv:2012.03755},
  year={2020}
}

@inproceedings{chen2022quantum,
  title={Quantum long short-term memory},
  author={Chen, Samuel Yen-Chi and Yoo, Shinjae and Fang, Yao-Lung L},
  booktitle={ICASSP 2022-2022 IEEE International Conference on Acoustics, Speech and Signal Processing (ICASSP)},
  pages={8622--8626},
  year={2022},
  organization={IEEE}
}

@inproceedings{zhang2018end,
  title={End-to-end quantum-like language models with application to question answering},
  author={Zhang, Peng and Niu, Jiabin and Su, Zhan and Wang, Benyou and Ma, Liqun and Song, Dawei},
  booktitle={Proceedings of the AAAI conference on artificial intelligence},
  volume={32},
  number={1},
  year={2018}
}

@inproceedings{qiao2024quantum2,
  title={A quantum-inspired matching network with linguistic theories for metaphor detection},
  author={Qiao, Wenbo and Zhang, Peng and Ma, ZengLai},
  booktitle={Proceedings of the 2024 joint international conference on computational linguistics, language resources and evaluation (LREC-cOLING 2024)},
  pages={1435--1445},
  year={2024}
}

@inproceedings{fan2024quantum,
  title={Quantum-inspired neural network with runge-kutta method},
  author={Fan, Zipeng and Zhang, Jing and Zhang, Peng and Lin, Qianxi and Gao, Hui},
  booktitle={Proceedings of the AAAI Conference on Artificial Intelligence},
  volume={38},
  number={16},
  pages={17977--17984},
  year={2024}
}

@inproceedings{jiang2020quantum,
  title={A quantum interference inspired neural matching model for ad-hoc retrieval},
  author={Jiang, Yongyu and Zhang, Peng and Gao, Hui and Song, Dawei},
  booktitle={Proceedings of the 43rd international ACM SIGIR conference on research and development in information retrieval},
  pages={19--28},
  year={2020}
}

@inproceedings{zhang2018quantum,
  title={A quantum many-body wave function inspired language modeling approach},
  author={Zhang, Peng and Su, Zhan and Zhang, Lipeng and Wang, Benyou and Song, Dawei},
  booktitle={Proceedings of the 27th ACM international conference on information and knowledge management},
  pages={1303--1312},
  year={2018}
}

@article{jones2019quest,
  title={QuEST and high performance simulation of quantum computers},
  author={Jones, Tyson and Brown, Anna and Bush, Ian and Benjamin, Simon C},
  journal={Scientific reports},
  volume={9},
  number={1},
  pages={10736},
  year={2019},
  publisher={Nature Publishing Group UK London}
}

@article{gu2023mamba,
  title={Mamba: Linear-time sequence modeling with selective state spaces},
  author={Gu, Albert and Dao, Tri},
  journal={arXiv preprint arXiv:2312.00752},
  year={2023}
}

@article{vaswani2017attention,
  title={Attention is all you need},
  author={Vaswani, Ashish and Shazeer, Noam and Parmar, Niki and Uszkoreit, Jakob and Jones, Llion and Gomez, Aidan N and Kaiser, {\L}ukasz and Polosukhin, Illia},
  journal={Advances in neural information processing systems},
  volume={30},
  year={2017}
}

@inproceedings{devlin2019bert,
  title={Bert: Pre-training of deep bidirectional transformers for language understanding},
  author={Devlin, Jacob and Chang, Ming-Wei and Lee, Kenton and Toutanova, Kristina},
  booktitle={Proceedings of the 2019 conference of the North American chapter of the association for computational linguistics: human language technologies, volume 1 (long and short papers)},
  pages={4171--4186},
  year={2019}
}

@article{hu2022lora,
  title={Lora: Low-rank adaptation of large language models.},
  author={Hu, Edward J and Shen, Yelong and Wallis, Phillip and Allen-Zhu, Zeyuan and Li, Yuanzhi and Wang, Shean and Wang, Lu and Chen, Weizhu and others},
  journal={ICLR},
  volume={1},
  number={2},
  pages={3},
  year={2022}
}

@article{hornik1991approximation,
  title={Approximation capabilities of multilayer feedforward networks},
  author={Hornik, Kurt},
  journal={Neural networks},
  volume={4},
  number={2},
  pages={251--257},
  year={1991},
  publisher={Elsevier}
}

@article{liu2024quantum,
  title={A quantum circuit-based compression perspective for parameter-efficient learning},
  author={Liu, Chen-Yu and Yang, Chao-Han Huck and Goan, Hsi-Sheng and Hsieh, Min-Hsiu},
  journal={arXiv preprint arXiv:2410.09846},
  year={2024}
}

@article{kong2025quantum,
  title={Quantum-enhanced llm efficient fine tuning},
  author={Kong, Xiaofei and Li, Lei and Chen, Zhaoyun and Xue, Cheng and Xu, Xiaofan and Liu, Huanyu and Wu, Yuchun and Fang, Yuan and Fang, Han and Chen, Kejiang and others},
  journal={arXiv preprint arXiv:2503.12790},
  year={2025}
}

@article{chen2024quanta,
  title={Quanta: Efficient high-rank fine-tuning of llms with quantum-informed tensor adaptation},
  author={Chen, Zhuo and Dangovski, Rumen and Loh, Charlotte and Dugan, Owen and Luo, Di and Soljacic, Marin},
  journal={Advances in Neural Information Processing Systems},
  volume={37},
  pages={92210--92245},
  year={2024}
}

@article{sanh2019distilbert,
  title={DistilBERT, a distilled version of BERT: smaller, faster, cheaper and lighter},
  author={Sanh, Victor and Debut, Lysandre and Chaumond, Julien and Wolf, Thomas},
  journal={arXiv preprint arXiv:1910.01108},
  year={2019}
}

@inproceedings{liu2024dora,
  title={Dora: Weight-decomposed low-rank adaptation},
  author={Liu, Shih-Yang and Wang, Chien-Yi and Yin, Hongxu and Molchanov, Pavlo and Wang, Yu-Chiang Frank and Cheng, Kwang-Ting and Chen, Min-Hung},
  booktitle={Forty-first International Conference on Machine Learning},
  year={2024}
}

@article{van2014scikit,
  title={scikit-image: image processing in Python},
  author={Van der Walt, Stefan and Sch{\"o}nberger, Johannes L and Nunez-Iglesias, Juan and Boulogne, Fran{\c{c}}ois and Warner, Joshua D and Yager, Neil and Gouillart, Emmanuelle and Yu, Tony},
  journal={PeerJ},
  volume={2},
  pages={e453},
  year={2014},
  publisher={PeerJ Inc.}
}

@article{le2025symmetry,
  title={Symmetry-invariant quantum machine learning force fields},
  author={Le, Isabel Nha Minh and Kiss, Oriel and Schuhmacher, Julian and Tavernelli, Ivano and Tacchino, Francesco},
  journal={New Journal of Physics},
  volume={27},
  number={2},
  pages={023015},
  year={2025},
  publisher={IOP Publishing}
}

@article{lu2020quantum,
  title={Quantum adversarial machine learning},
  author={Lu, Sirui and Duan, Lu-Ming and Deng, Dong-Ling},
  journal={Physical Review Research},
  volume={2},
  number={3},
  pages={033212},
  year={2020},
  publisher={APS}
}

@article{dong2008quantum,
  title={Quantum reinforcement learning},
  author={Dong, Daoyi and Chen, Chunlin and Li, Hanxiong and Tarn, Tzyh-Jong},
  journal={IEEE Transactions on Systems, Man, and Cybernetics, Part B (Cybernetics)},
  volume={38},
  number={5},
  pages={1207--1220},
  year={2008},
  publisher={IEEE}
}

@article{cong2019quantum,
  title={Quantum convolutional neural networks},
  author={Cong, Iris and Choi, Soonwon and Lukin, Mikhail D},
  journal={Nature Physics},
  volume={15},
  number={12},
  pages={1273--1278},
  year={2019},
  publisher={Nature Publishing Group UK London}
}

@article{li2024quantum,
  title={Quantum gated recurrent neural networks},
  author={Li, Yanan and Wang, Zhimin and Xing, Ruipeng and Shao, Changheng and Shi, Shangshang and Li, Jiaxin and Zhong, Guoqiang and Gu, Yongjian},
  journal={IEEE Transactions on Pattern Analysis and Machine Intelligence},
  year={2024},
  publisher={IEEE}
}

@article{pan2022simulation,
  title={Simulation of quantum circuits using the big-batch tensor network method},
  author={Pan, Feng and Zhang, Pan},
  journal={Physical Review Letters},
  volume={128},
  number={3},
  pages={030501},
  year={2022},
  publisher={APS}
}

@article{shi2024pretrained,
  title={Pretrained quantum-inspired deep neural network for natural language processing},
  author={Shi, Jinjing and Chen, Tian and Lai, Wei and Zhang, Shichao and Li, Xuelong},
  journal={IEEE Transactions on Cybernetics},
  volume={54},
  number={10},
  pages={5973--5985},
  year={2024},
  publisher={IEEE}
}

@article{yan2021quantum,
  title={Quantum probability-inspired graph neural network for document representation and classification},
  author={Yan, Peng and Li, Linjing and Jin, Miaotianzi and Zeng, Daniel},
  journal={Neurocomputing},
  volume={445},
  pages={276--286},
  year={2021},
  publisher={Elsevier}
}

@inproceedings{li2021quantum,
  title={Quantum-inspired neural network for conversational emotion recognition},
  author={Li, Qiuchi and Gkoumas, Dimitris and Sordoni, Alessandro and Nie, Jian-Yun and Melucci, Massimo},
  booktitle={Proceedings of the AAAI conference on artificial intelligence},
  volume={35},
  number={15},
  pages={13270--13278},
  year={2021}
}

@inproceedings{xiong2024node2ket,
  title={Node2ket: Efficient high-dimensional network embedding in quantum hilbert space},
  author={Xiong, Hao and Tang, Yehui and He, Yunlin and Tan, Wei and Yan, Junchi},
  booktitle={The Twelfth International Conference on Learning Representations},
  year={2024}
}

@inproceedings{yan2022towards,
  title={Towards a native quantum paradigm for graph representation learning: A sampling-based recurrent embedding approach},
  author={Yan, Ge and Tang, Yehui and Yan, Junchi},
  booktitle={Proceedings of the 28th ACM SIGKDD Conference on Knowledge Discovery and Data Mining},
  pages={2160--2168},
  year={2022}
}

@article{bai2025aegk,
  title={AEGK: Aligned Entropic Graph Kernels Through Continuous-Time Quantum Walks},
  author={Bai, Lu and Cui, Lixin and Li, Ming and Ren, Peng and Wang, Yue and Zhang, Lichi and Yu, Philip S and Hancock, Edwin R},
  journal={IEEE Transactions on Knowledge and Data Engineering},
  year={2025},
  publisher={IEEE}
}

@article{koike2025quantum,
  title={Quantum-PEFT: Ultra parameter-efficient fine-tuning},
  author={Koike-Akino, Toshiaki and Tonin, Francesco and Wu, Yongtao and Wu, Frank Zhengqing and Candogan, Leyla Naz and Cevher, Volkan},
  journal={arXiv preprint arXiv:2503.05431},
  year={2025}
}

@article{anschuetz2023interpretable,
  title={Interpretable quantum advantage in neural sequence learning},
  author={Anschuetz, Eric R and Hu, Hong-Ye and Huang, Jin-Long and Gao, Xun},
  journal={PRX Quantum},
  volume={4},
  number={2},
  pages={020338},
  year={2023},
  publisher={APS}
}

@article{bruza2015quantum,
  title={Quantum cognition: a new theoretical approach to psychology},
  author={Bruza, Peter D and Wang, Zheng and Busemeyer, Jerome R},
  journal={Trends in cognitive sciences},
  volume={19},
  number={7},
  pages={383--393},
  year={2015},
  publisher={Elsevier}
}

@inproceedings{zhang2024quantum,
  title={Quantum interference model for semantic biases of glosses in word sense disambiguation},
  author={Zhang, Junwei and He, Ruifang and Guo, Fengyu and Liu, Chang},
  booktitle={Proceedings of the AAAI Conference on Artificial Intelligence},
  volume={38},
  number={17},
  pages={19551--19559},
  year={2024}
}

@article{nguyen2022evaluation,
  title={An evaluation of hardware-efficient quantum neural networks for image data classification},
  author={Nguyen, Tuyen and Paik, Incheon and Watanobe, Yutaka and Thang, Truong Cong},
  journal={Electronics},
  volume={11},
  number={3},
  pages={437},
  year={2022},
  publisher={MDPI}
}

@article{fan2024quantum2,
  title={Quantum-inspired language models based on unitary transformation},
  author={Fan, Zipeng and Zhang, Jing and Zhang, Peng and Lin, Qianxi and Li, Yizhe and Qian, Yuhua},
  journal={Information Processing \& Management},
  volume={61},
  number={4},
  pages={103741},
  year={2024},
  publisher={Elsevier}
}

@article{wang2025predictive,
  title={Predictive Performance of Deep Quantum Data Re-uploading Models},
  author={Wang, Xin and Tao, Han-Xiao and Wu, Re-Bing},
  journal={arXiv preprint arXiv:2505.20337},
  year={2025}
}

@inproceedings{van2018semeval,
  title={Semeval-2018 task 3: Irony detection in english tweets},
  author={Van Hee, Cynthia and Lefever, Els and Hoste, V{\'e}ronique},
  booktitle={Proceedings of The 12th International Workshop on Semantic Evaluation},
  pages={39--50},
  year={2018}
}

@article{jiang2023diego,
  title={Diego de las Casas},
  author={Jiang, Albert Q and Sablayrolles, Alexandre and Mensch, Arthur and Bamford, Chris and Chaplot, Devendra Singh},
  journal={Florian Bressand, Gianna Lengyel, Guillaume Lample, Lucile Saulnier, L{\'e}lio Renard Lavaud, Marie-Anne Lachaux, Pierre Stock, Teven Le Scao, Thibaut Lavril, Thomas Wang, Timoth{\'e}e Lacroix, and William El Sayed},
  pages={50--72},
  year={2023}
}

@article{bergholm2018pennylane,
  title={Pennylane: Automatic differentiation of hybrid quantum-classical computations},
  author={Bergholm, Ville and Izaac, Josh and Schuld, Maria and Gogolin, Christian and Ahmed, Shahnawaz and Ajith, Vishnu and Alam, M Sohaib and Alonso-Linaje, Guillermo and AkashNarayanan, B and Asadi, Ali and others},
  journal={arXiv preprint arXiv:1811.04968},
  year={2018}
}

@String{Computing = "Computing" }

@ArtifactSoftware{R,
    title = {R: A Language and Environment for Statistical Computing},
    author = {{R Core Team}},
    organization = {R Foundation for Statistical Computing},
    address = {Vienna, Austria},
    year = {2019},
    url = {https://www.R-project.org/},
}

@inproceedings{han2024quantum,
  title={Quantum cognition-inspired EEG-based recommendation via graph neural networks},
  author={Han, Jinkun and Li, Wei and Li, Yingshu and Cai, Zhipeng},
  booktitle={Proceedings of the 33rd ACM International Conference on Information and Knowledge Management},
  pages={778--788},
  year={2024}
}

@article{zhang2022complex,
  title={Complex-valued neural network-based quantum language models},
  author={Zhang, Peng and Hui, Wenjie and Wang, Benyou and Zhao, Donghao and Song, Dawei and Lioma, Christina and Simonsen, Jakob Grue},
  journal={ACM Transactions on Information Systems (TOIS)},
  volume={40},
  number={4},
  pages={1--31},
  year={2022},
  publisher={ACM New York, NY}
}

@inproceedings{wang2019qpin,
  title={QPIN: a quantum-inspired preference interactive network for E-commerce recommendation},
  author={Wang, Panpan and Li, Zhao and Zhang, Yazhou and Hou, Yuexian and Ge, Liangzhu},
  booktitle={Proceedings of the 28th ACM International Conference on Information and Knowledge Management},
  pages={2329--2332},
  year={2019}
}

@inproceedings{shokrollahi2023intersectional,
  title={Intersectional bias mitigation in pre-trained language models: A quantum-inspired approach},
  author={Shokrollahi, Omid},
  booktitle={Proceedings of the 32nd ACM International Conference on Information and Knowledge Management},
  pages={5181--5184},
  year={2023}
}

@inproceedings{tian2020qsan,
  title={QSAN: A quantum-probability based signed attention network for explainable false information detection},
  author={Tian, Tian and Liu, Yudong and Yang, Xiaoyu and Lyu, Yuefei and Zhang, Xi and Fang, Binxing},
  booktitle={Proceedings of the 29th ACM international conference on information \& knowledge management},
  pages={1445--1454},
  year={2020}
}

@inproceedings{chen2024hands,
  title={Hands-On Introduction to Quantum Machine Learning},
  author={Chen, Samuel Yen-Chi and Kim, Joongheon},
  booktitle={Proceedings of the 33rd ACM International Conference on Information and Knowledge Management},
  pages={5507--5510},
  year={2024}
}

@inproceedings{baek2023logarithmic,
  title={Logarithmic dimension reduction for quantum neural networks},
  author={Baek, Hankyul and Park, Soohyun and Kim, Joongheon},
  booktitle={Proceedings of the 32nd ACM International Conference on Information and Knowledge Management},
  pages={3738--3742},
  year={2023}
}

@inproceedings{ferrari2024using,
  title={Using and Evaluating Quantum Computing for Information Retrieval and Recommender Systems},
  author={Ferrari Dacrema, Maurizio and Pasin, Andrea and Cremonesi, Paolo and Ferro, Nicola},
  booktitle={Proceedings of the 47th International ACM SIGIR Conference on Research and Development in Information Retrieval},
  pages={3017--3020},
  year={2024}
}

@inproceedings{ferrari2022towards,
  title={Towards feature selection for ranking and classification exploiting quantum annealers},
  author={Ferrari Dacrema, Maurizio and Moroni, Fabio and Nembrini, Riccardo and Ferro, Nicola and Faggioli, Guglielmo and Cremonesi, Paolo},
  booktitle={Proceedings of the 45th International ACM SIGIR Conference on Research and Development in Information Retrieval},
  pages={2814--2824},
  year={2022}
}
\appendix

\newpage
\section{The proof of Claim 2}
\label{app1}

\noindent \textbf{Claim 2.}
\textit{Under optimal conditions, Q-RUN is capable of expressing a $d$-dimensional truncated Fourier series with a frequency spectrum containing up to $3^{dn}$ components.}

\begin{proof}[Proof]

We analyze Q-RUN through parameterized evolution and measurement. The model is expressed as:
\begin{equation}\label{Ae3}
f(\boldsymbol{x}) = \bra{0}^{\otimes nd} \, S(\boldsymbol{x})^\dagger \, \hat{\boldsymbol{O}} \, S(\boldsymbol{x}) \ket{0}^{\otimes nd}
\end{equation}
where $S(\boldsymbol{x})$ denotes the encoding operator, and $\hat{\boldsymbol{O}}$ represents a Hermitian observable. The encoding operator repeatedly uploads each element of \(\boldsymbol{x}\) into the circuit \(n\) times. We adopt a multi-qubit, single-layer encoding scheme, where each input element is embedded into an \(n\)-qubit subsystem. Accordingly, the overall encoding takes the form:
\begin{equation}
    S(\boldsymbol{x}) = e^{-i x_1 H} \otimes \cdots \otimes e^{-i x_d H},
    \label{eq:encoding_unitary}
\end{equation}
where each \(H\) is a local Hamiltonian acting on its corresponding \(n\)-qubit subsystem.

Since \( H \) can always be diagonalized as \( H := V \Sigma V^\dagger \), where \( \Sigma \) is a diagonal matrix and \( V \) is a unitary matrix. We observe that the left unitary \( V^\dagger \) can be absorbed into the observable \( \boldsymbol{O} \), while the right unitary \( V \) can be merged into the initial state \( \ket{0}^{\otimes nd} \). Thus, this is equivalent to preparing a fixed state \( \ket{\Gamma} \). This state can be expanded in the computational basis as:
\begin{equation}
    \ket{\Gamma} = \sum_{\boldsymbol{j}} \gamma_{\boldsymbol{j}} \ket{\boldsymbol{j}}, \quad \text{with } \boldsymbol{j} = (j_1, \dots, j_d),
    \label{e4}
\end{equation}
where each index \( j_k \in \{1, \dots, 2^n\} \) corresponds to the local Hilbert space of dimension \( 2^n \).

Since the surrounding unitary matrices \( V \) and \( V^\dagger \) have been absorbed, the central diagonal matrix \( \Sigma \) can be directly regarded as a diagonalized Hamiltonian.
\begin{equation}
    H := \Sigma = \mathrm{diag}(\lambda_1, \dots, \lambda_{2^n}),
    \label{eq:diagonal_hamiltonians}
\end{equation}
so that the overall encoding $S(\boldsymbol{x})$ is also diagonal, with elements:
\begin{equation}
    [S(\boldsymbol{x})]_{\boldsymbol{j},\boldsymbol{j}} = \exp\left(-i \boldsymbol{x} \cdot \boldsymbol{\lambda}_{\boldsymbol{j}}\right),
    \label{e5}
\end{equation}
where the spectral vector $\boldsymbol{\lambda}_{\boldsymbol{j}} := (\lambda_{j_1}, \dots, \lambda_{j_d})$ collects eigenvalues corresponding to the basis state $|\boldsymbol{j}\rangle$.

By substituting Eq. (\ref{e4}) and Eq. (\ref{e5}) into Eq. (\ref{Ae3}), the output of Q-RUN becomes:
\begin{equation}
f(\boldsymbol{x}) = \sum_{\boldsymbol{j},\boldsymbol{k}} \gamma_{\boldsymbol{j}}^* \gamma_{\boldsymbol{k}}\hat{\mathbf{O}}_{\boldsymbol{j},\boldsymbol{k}} \, e^{i \boldsymbol{x} \cdot (\boldsymbol{\lambda_k} - \boldsymbol{\lambda_j})}
= \sum_{\boldsymbol{\omega} \in \Omega} c_{\boldsymbol{\omega}} e^{i \boldsymbol{\omega} \cdot \boldsymbol{x}},
\end{equation}
which is a multivariate Fourier-type expansion. The set of available frequencies $\Omega$ is dictated by the pairwise differences of spectral vectors $\boldsymbol{\lambda_k} - \boldsymbol{\lambda_j}$, while the amplitude of each Fourier component is governed by the parameters $\gamma$ and entries of the observable $\hat{\mathbf{O}}$.

Furthermore, we aim to characterize the capability of Q-RUN to approximate Fourier series, which is inherently determined by the spectrum of frequencies it can represent. Specifically, Q-RUN applies learnable parameter vectors \( w = [w_1, \ldots, w_n] \) to each element of the input vector \( \boldsymbol{x} \), followed by $R_y$ gates to encode the scaled input. This procedure yields a specific $n$-qubit Hamiltonian:
\[
H^* = \sum_{i=1}^n \frac{w_i}{2} \mathrm{Y}^{(i)},
\]
where ${Y}^{(i)} := (\otimes_{j=1}^{i-1} I) \otimes {Y} \otimes (\otimes_{j=i+1}^{N} \mathrm{I})$. ${Y}$ denoting the Pauli-$Y$ operator and $\mathrm{I}$ being the $2 \times 2$ identity matrix.

The eigenvalues of $H^*$ determine the available frequency components $\boldsymbol{{\lambda^*}_i}$ in Q-RUN. 
Since the eigenvalues of the Pauli-Y gate are known to be \(\lambda = \pm 1\), it naturally follows that each time data is uploaded, the resulting operator \(H^*\) can produce the following eigenvalues:
\begin{equation}\label{e11}
\begin{aligned}
&\text{When} \quad H^* \in \mathbb{C}^{2^1 \times 2^1}, \\
&\boldsymbol{{\lambda^*}_i} \in \left\{ \boldsymbol{\lambda_k - \lambda_j} \;\middle|\;\boldsymbol{\lambda_k}, \boldsymbol{\lambda_j} \in \left\{ -\frac{w_{1}}{2}, +\frac{w_{1}}{2} \right\} \right\}, \\
&\boldsymbol{{\lambda^*}_i}\in \Omega^{(1)} = \left\{ -w_{1},\; 0,\; +w_{1} \right\}; \\
&\text{When} \quad H^* \in \mathbb{C}^{2^2 \times 2^2}, \\
&\boldsymbol{{\lambda^*}_i} \in \left\{ \boldsymbol{\lambda_k - \lambda_j} \;\middle|\; \boldsymbol{\lambda_k}, \boldsymbol{\lambda_j} \in \left\{ \pm\frac{w_{1}}{2},\; \pm\frac{w_{2}}{2} \right\} \right\}, \\
&\boldsymbol{{\lambda^*}_i}\in \Omega^{(2)} = \{ 
-w_{2} - w_{1},\; -w_{2},\; -w_{2} + w_{1},\; -w_{1},\; 0, 
 w_{1},\; w_{2} - w_{1},\; w_{2},\; w_{2} + w_{1} \}; \\
&\vdots \\
&\text{When} \quad H^* \in \mathbb{C}^{2^n \times 2^n}, \boldsymbol{{\lambda^*}_i} \in \Omega^{(n)} = \left\{ 
\Omega^{(n-1)} - w_{n},\; 
\Omega^{(n-1)},\; 
\Omega^{(n-1)} + w_{n}
\right\}.
\end{aligned}
\end{equation}

Therefore, for a Q-RUN model with \(n\) re-uploading times and \(d\)-dimensional input, the spectrum that can be covered is:
\begin{equation}\label{e12}
\begin{aligned}
&\left\{ \boldsymbol{\lambda_k} - \boldsymbol{\lambda_j}\right\}_{\boldsymbol{k}, \boldsymbol{j}} 
= \left\{  {\boldsymbol{{\lambda^*}_i}} \;\middle|\; \boldsymbol{{\lambda^*}_i} \in \Omega_i^{(n)} \right\}_{i=1}^d, \text{where} \quad \Omega^{(n)}_i = \left\{ \Omega^{(n-1)}_i - w_{n},\; \Omega^{(n-1)}_i,\; \Omega^{(q-1)}_i + w_{n} \right\}, \Omega^{(1)}_i = \left\{ -w_{1},\; 0,\; w_{1} \right\}.
\end{aligned}
\end{equation}
Since the parameter vector \( \boldsymbol{w} \) is learnable, a well-trained \( \boldsymbol{w} \) can ensure that the resulting frequency components are non-degenerate. In the best case, the number of expressible frequencies can reach:
\[
\left| \left\{ \boldsymbol{\lambda}_{\boldsymbol{k}} - \boldsymbol{\lambda}_{\boldsymbol{j}} \right\} \right| = 3^{d n}.
\]

In QML models, the number of learnable parameters typically scales with the number of qubits, i.e., it is on the order of \( \mathcal{O}(nd) \). Q-RUN can potentially cover an exponential number of frequency components with respect to its linearly growing number of parameters, indicating that its Fourier fitting capacity may surpass that of classical methods and enabling it to approximate high-frequency functions using significantly fewer parameters.
\end{proof}

\end{document}